%% file: main.tex
\documentclass{article}

\PassOptionsToPackage{table}{xcolor}

 \usepackage[main, final]{neurips_2026}

\usepackage[utf8]{inputenc} 
\usepackage[T1]{fontenc}    
\usepackage{nameref,zref-xr}
\usepackage{hyperref}       
\usepackage{url}            
\usepackage{booktabs}       
\usepackage{amsfonts}       
\usepackage{nicefrac}       
\usepackage{xcolor}         
\usepackage{amsmath,amssymb,amsthm}

\usepackage[most]{tcolorbox}
\usepackage{amssymb}
\usepackage{mathtools}
\usepackage{amsthm}
\usepackage{wrapfig}
\usepackage{caption} 
\usepackage{import}
\usepackage{xspace}
\usepackage[inline]{enumitem}
\usepackage{algorithm}
\usepackage{nicefrac}       
\usepackage{microtype} 
\usepackage{multirow}
\usepackage{graphicx}
\usepackage{adjustbox}
\usepackage{tabularx}
\usepackage{hhline}
\usepackage{makecell}
\usepackage{arydshln}
\usepackage{array}
\usepackage{xcolor}
\usepackage{color, colortbl}
\usepackage{algpseudocode}
\usepackage{wrapfig}
\definecolor{Gray}{gray}{0.9}
\definecolor{light_purple}{RGB}{235,236,242}
\definecolor{light_blue}{RGB}{244,249,254}
\definecolor{light_origin}{RGB}{207, 117, 58}
\definecolor{cell_purple}{RGB}{76, 141, 187}
\definecolor{small_blue}{RGB}{27, 63, 125}
\definecolor{small_origin}{RGB}{94, 41, 16}
\definecolor{mypink}{HTML}{D92879}
\definecolor{bgcolor}{RGB}{242, 243, 245}
\definecolor{customcolor}{RGB}{76, 141, 187}
\definecolor{customcolor2}{RGB}{207, 117, 58}

\newtheorem{theorem}{Theorem}
\makeatletter
\def\adl@drawiv#1#2#3{%
        \hskip.5\tabcolsep
        \xleaders#3{#2.5\@tempdimb #1{1}#2.5\@tempdimb}%
                #2\z@ plus1fil minus1fil\relax
        \hskip.5\tabcolsep}
\newcommand{\cdashlinelr}[1]{%
  \noalign{\vskip\aboverulesep
            \global\let\@dashdrawstore\adl@draw
            \global\let\adl@draw\adl@drawiv}
  \cdashline{#1}
  \noalign{\global\let\adl@draw\@dashdrawstore
            \vskip\belowrulesep}}
\makeatother

\usepackage{babel}
\providecommand{\lemmaname}{Lemma}
\providecommand{\theoremname}{Theorem}

\theoremstyle{plain}

\theoremstyle{plain}

\newcolumntype{a}{>{\columncolor{Gray}}c}
\newcolumntype{B}{>{\columncolor{blue!4}}l}
\newcolumntype{d}{>{\columncolor{brown!4}}l}
\newcolumntype{q}{>{\columncolor{green!2}}c}
\newcolumntype{H}{>{\setbox0=\hbox\bgroup}c<{\egroup}@{}}
\newcolumntype{R}{>{\columncolor{red!4}}l}
\newcolumntype{P}{>{\columncolor{purple!4}}c}
\newcolumntype{Y}{>{\columncolor{yellow!4}}c}
\newcolumntype{Z}{>{\setbox0=\hbox\bgroup}c<{\egroup}@{\hspace*{-\tabcolsep}}}

\makeatletter
\@ifundefined{proposition}{\newtheorem{proposition}[theorem]{Proposition}}{}
\@ifundefined{corollary}{\newtheorem{corollary}[theorem]{Corollary}}{}
\@ifundefined{lemma}{\newtheorem{lemma}[theorem]{Lemma}}{}
\makeatother

\usepackage{babel}
\providecommand{\lemmaname}{Lemma}
\providecommand{\theoremname}{Theorem}

\title{Rationale-Guided Policy Optimization: Learning to Reason with Adaptive Rationale Scaffolding}

\newcommand{\RGPO}{\textsc{RGPO}}

\tcbuselibrary{theorems,skins,breakable}

\newtcbtheorem{Example}{\bfseries Example}{
  enhanced,
  breakable,
  drop shadow={black!50!white},
  coltitle=black,
  top=0.3in,
  attach boxed title to top left={
    xshift=1.5em,
    yshift=-\tcboxedtitleheight/2
  },
  varwidth boxed title=\dimexpr\linewidth-3em\relax,
  boxed title style={
    size=small,
    colback=pink,
    halign=left
  }
}{Example}
\definecolor{originalquestion}{HTML}{8C564B} 
\definecolor{partialsolution}{HTML}{FF7F0E}  
\definecolor{previousresponse}{HTML}{1F77B4}  
\definecolor{judgefeedback}{HTML}{D62728}     
\definecolor{instruction}{HTML}{9467BD}       
\definecolor{refinedanswer}{HTML}{2CA02C}   

\author{%
  Hoang Phan \\
  New York University\\
  \texttt{hvp2011@nyu.edu} \\
  \And
  Minh Pham \\
  New York University\\
  \texttt{mnpham@nyu.edu} \\
  \And
  Chau Pham \\
  University at Buffalo\\
  \texttt{haichaup@buffalo.edu} \\
  \AND
  Chinmay Hegde \\
  New York University\\
  \texttt{chinmay.h@nyu.edu} \\
  \And
  Trung Le \\
  Monash University\\
  \texttt{trunglm@monash.edu} \\
  \And  
  Qi Lei \\
  New York University\\
  \texttt{ql518@nyu.edu} 
}

\begin{document}

\maketitle

\input{sec/0_abstract}

\input{sec/1_intro}

\input{sec/2_related_work}
\input{sec/3_background}
\input{sec/4_method}

\input{sec/5_experiment}
\input{sec/6_conclusion}


\section*{Acknowledgments}
QL acknowledges support of NSF DMS-2523382 and DOE Office of Science under Award \#DESC0024721. 

\bibliographystyle{unsrt}
\bibliography{ref}

\newpage
\appendix
\input{sec/7_supp}
\end{document}

%% file: sec/0_abstract.tex
\begin{abstract}
On-policy reinforcement learning has become a central paradigm for improving the reasoning abilities of large language models. However, its effectiveness is often limited by reward sparsity: when a model fails to discover correct trajectories for difficult problems, the optimization process receives little useful signal and may stagnate. Existing approaches mitigate this issue by incorporating off-policy demonstrations, expert traces, or model-generated solutions, but they typically require the auxiliary data to match the format of the reinforcement-learning task, often relying on rejection sampling from stronger models to obtain suitable training trajectories. We introduce \textbf{Rationale-Guided Policy Optimization (RGPO)}, a framework that adaptively leverages ground-truth rationale information according to the model's current capability while preserving its freedom to explore. Rather than treating reference solutions as fixed imitation targets, RGPO uses them as temporary scaffolds: rationales help the model generate improved responses, after which only higher-reward, model-generated solutions are transferred back to the original unguided setting. This design allows training to exploit available ground-truth information without requiring off-policy data to follow the same format as the RL task. Across both language-only and vision-language reasoning settings, RGPO consistently improves performance over RLVR baselines, and ablation studies show that adaptive rationale guidance is a key contributor to these gains. These results suggest that RGPO offers a practical and general approach for reducing reward sparsity, stabilizing reinforcement learning, and improving reasoning performance in both text-only and multimodal models. Our implementation is available at \url{https://github.com/VietHoang1512/rgpo}.
\end{abstract}

%% file: sec/1_intro.tex
\section{Introduction}

Recent advances in large language model reasoning have shown that reinforcement learning, particularly reinforcement learning with verifiable rewards (RLVR), can elicit long-horizon reasoning  behaviors from pretrained models. Current frontier models, including GPT-6 Astra~\citep{openai2026gpt6astra},
Gemini 4 Argon~\citep{google2026gemini4argon},
Claude Opus 5.5~\citep{anthropic2026opus55},
DeepSeek-V4.1-Flash~\citep{deepseek2026v41flash}, and
Kimi K3~\citep{moonshot2026kimik3}, have continued to advance
performance in mathematical reasoning, coding, tool use, and long-horizon
agentic tasks. Reinforcement learning has become an important component of
post-training for frontier reasoning systems, enabling models to improve
beyond supervised demonstrations through task-level feedback
\citep{schulman2017ppo,shao2024deepseekmath,openai2026gpt6astra,
deepseek2026v41flash}. Despite its effectiveness, on-policy reinforcement learning remains fundamentally constrained by the current capability of the policy model. When the model is trained on problems beyond its evolving reasoning ability, all sampled rollouts for a prompt may be incorrect. In such cases, rewards become sparse or uniform, and the resulting learning signal provides little information about how to reach a correct solution. This capacity--difficulty mismatch is especially severe for weak or small models, where standard on-policy RL may spend substantial computation sampling failed trajectories without obtaining positive examples for improvement \citep{yan2025luffy,liu2025ghpo,zhang2025bread}. Consequently, the training process can stagnate precisely on the hard examples that are most important for improving reasoning. A natural direction for mitigating reward sparsity is to incorporate external guidance, such as reference solutions or expert traces \citep{yan2025luffy,wang2025hint}. Prior work has explored several forms of such guidance. However, directly using external rationales also introduces important challenges. Maximizing the likelihood of those reference solutions may force the model to imitate trajectories that are far from its own policy distribution, causing distribution mismatch, memorization, and limited generalization \cite{chu2025sft}. Off-policy expert traces can provide successful solutions, but they may not reflect the reasoning style or token-level decisions that the current model can naturally reproduce.  

Human learning research suggests that effective reasoning instruction requires a careful balance between independent problem solving and guided assistance. Since initial struggle can prepare learners to better recognize, organize, and internalize later guidance \citep{kapur2008productive,schwartz1998time}, they can benefit from attempting a problem before receiving explicit instruction. At the same time, intelligent tutoring research identifies an “assistance dilemma”: too little help may leave learners stuck, whereas too much help may reduce effort, encourage shallow processing, and weaken transfer \citep{koedinger2007exploring}. Motivated by this perspective, we propose \textbf{Rationale-Guided Policy Optimization (RGPO)}, a simple yet effective framework that builds on the standard RLVR training loop for improving reasoning by using reference rationales as temporary scaffolds for exploration rather than as direct imitation targets.  For each prompt, the model first generates ordinary on-policy rollouts and receives rewards from the verifier. Then, for examples where guidance may be useful, RGPO constructs a second-round refinement prompt that includes the original problem together with a rationale, solution hint, or feedback derived from the training example. The model is asked to generate a new standalone solution to the original problem. If this yields a better response than the original rollout, RGPO applies a supervised update to maximize the likelihood of the refined response conditioned only on the original prompt, rather than on the hint-augmented prompt.

\begin{figure}[t] 
    \centering
     \includegraphics[width=.95\columnwidth]{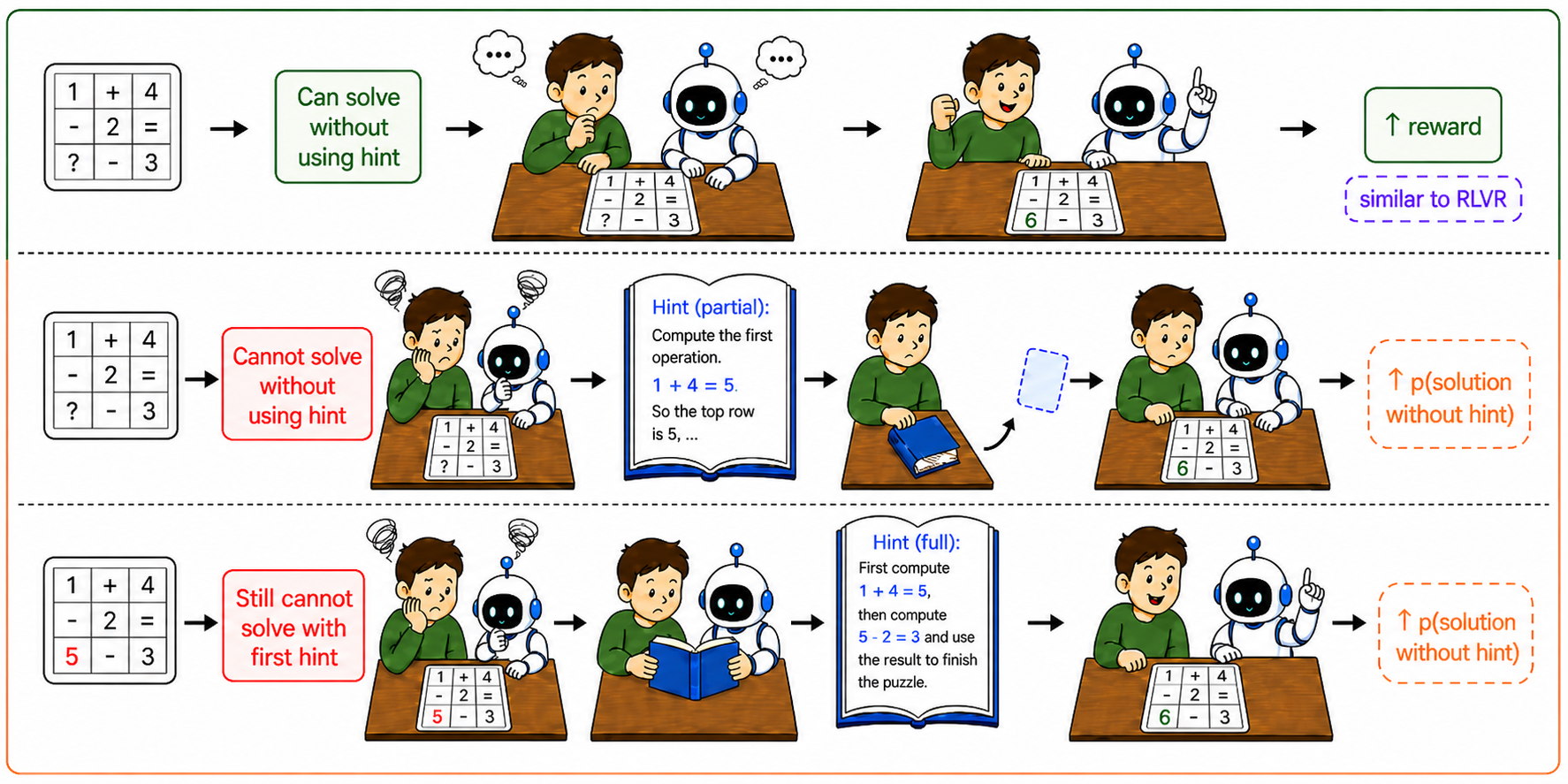}
    \vspace{-.2cm}
    \caption{\textbf{Overview of the proposed method:} Inspired by how humans approach a problem, the model first tries to solve each problem without access to hints. When it succeeds, the correct response is directly reinforced. When it fails, progressively stronger rationales are introduced as temporary scaffolds to guide the model toward a correct solution.}
    \label{fig:overview}
    \vspace{-.6cm}
\end{figure}

To further balance guidance and exploration, we introduce an adaptive hinting mechanism. Rather than always revealing the full reference solution, the scheduler adjusts the amount of guidance according to the model's current success on each problem (see Figure \ref{fig:overview}). If the model repeatedly fails to solve a problem, RGPO reveals more of the rationale to help it reach a successful trajectory. If the model can already solve the problem, the scheduler reduces or removes the hint, encouraging independent reasoning and avoiding unnecessary simplification. This adaptive mechanism allows RGPO to concentrate guidance on difficult examples while preserving on-policy exploration for problems within the model's current capability. Overall, our contributions are summarized as follows:
\begin{itemize}
    \item We propose \textbf{RGPO}, a rationale-guided optimization framework that augments RLVR with hint-assisted refinement. In particular, we show how reference rationales can be used as \emph{temporary reasoning scaffolds} rather than direct imitation targets, reducing reward sparsity without requiring the rationale to match the final-answer format used in RL rollouts.
    \item We introduce an adaptive hint mechanism that dynamically increases or decreases the amount of revealed solution guidance according to the model's current ability.
    \item We conduct experiments on both language models and vision-language models, where RGPO consistently shows superior training sample efficiency over standard RLVR and outperforms other hybrid baselines that combine both on-policy and off-policy training.
\end{itemize}

%% file: sec/2_related_work.tex
\section{Related Work}


\textbf{RL with verifiable rewards.}
PPO optimizes a clipped policy-gradient objective and is widely used for language-model post-training \citep{schulman2017ppo}. GRPO replaces value estimation with group-relative reward normalization, reducing training complexity for mathematical reasoning \citep{shao2024deepseekmath}. DeepSeek-R1 further demonstrated that large-scale RLVR can elicit sophisticated reasoning behaviors without relying entirely on supervised reasoning traces \citep{deepseekai2025deepseekr1}. \RGPO{} is compatible with either PPO or GRPO: it leaves the first-round RL update intact and adds a second-round reward-gated distillation term.

\textbf{Off-policy guidance and mixed SFT--RL training.}
A major limitation of pure on-policy RL is that the policy can only learn from trajectories it already samples. LUFFY addresses this by incorporating off-policy reasoning traces while using policy shaping to avoid rigid imitation \citep{yan2025luffy}. CHORD reframes SFT as a dynamically weighted auxiliary objective within RL \citep{zhang2025chord}. Minimalist rejection-sampling and reinforce-style methods also show that filtering positive trajectories can be a strong baseline \citep{xiong2025minimalist}. \RGPO{} differs in that its auxiliary examples are generated online by the current policy under rationale-guided refinement, then trained under the original prompt distribution. GHPO adjusts the amount of guidance according to the model's ability to solve the problem \citep{liu2025ghpo}. QuestA augments hard questions with partial solutions \citep{li2025questa}. BREAD inserts expert prefixes and branches rollouts from them \citep{zhang2025bread}. HINT supplies heuristic hints and argues that external guidance should preserve training affinity \citep{wang2025hint}. 


\textbf{Failure-conditioned guidance.} The design of RGPO is primarily motivated by cognitive studies on when and how guidance should be provided. Effective reasoning instruction depends on both the timing and form of guidance, as learners often benefit from attempting problems before receiving instruction, because initial struggle can help them notice relevant structures, activate prior knowledge, and better internalize subsequent explanations \citep{kapur2008productive,schwartz1998time}. Besides, active recall and testing improve long-term retention more effectively than passive study, motivating an initial phase of independent answer generation \citep{roediger2006test}. The ``assistance dilemma'' \citep{koedinger2007exploring} states that too little help may leave learners stuck, while too much help may reduce effort, encourage shallow processing, and weaken transfer. Studies of help seeking also show that learners may engage in hint abuse, using assistance to obtain answers rather than to understand the underlying reasoning process \citep{aleven2004toward}. Complementary work on worked examples demonstrates that examples can support schema acquisition, but their benefits depend on active engagement, such as self-explanation, rather than copying solution steps \citep{sweller1985use,chi1989self}. Finally, fading worked-out steps helps learners transition from guided examples to independent problem solving, suggesting that explanations and hints are most useful as temporary scaffolds rather than permanent supports \citep{atkinson2003transitioning}. 

%% file: sec/3_background.tex
\section{Background}
\label{sec:background}

\subsection{Supervised Fine-Tuning}
\label{sec:background_sft}

Supervised fine-tuning (SFT) is one of the most widely used post-training
procedures for adapting a pretrained language model to a target task, domain, or
response format. Given a dataset of prompt--response pairs
$\mathcal{D}_{\mathrm{sft}}=\{(x_i,y_i)\}_{i=1}^{N}$, SFT optimizes the model by
maximizing the conditional likelihood of the reference response:
\begin{equation}
\mathcal{L}_{\mathrm{sft}}(\theta)
=
-\mathbb{E}_{(x,y)\sim \mathcal{D}_{\mathrm{sft}}}
\left[
\sum_{t=1}^{|y|}
\log \pi_{\theta}(y_t \mid x, y_{<t})
\right].
\end{equation}
In instruction-following pipelines, SFT is commonly used to initialize the policy
before preference-based or reward-based optimization. For example,
Ouyang et al. \citep{ouyang2022training} first fine-tuned a language model on human-written
demonstrations and then further optimized it using reinforcement learning from
human feedback.

In reasoning tasks, SFT can teach the model answer formatting, step-by-step
reasoning style, and task-specific conventions. However, human-written or oracle solutions may be concise or stylistically different from model-generated reasoning (i.e. outside the
model's current reasoning distribution). As a result, simply maximizing the
likelihood of such demonstrations can lead to superficial imitation or
memorization \citep{chu2025sft}, rather than improving the model's ability to independently discover
correct reasoning trajectories. This limitation can be particularly pronounced for smaller models when
expert traces impose a reasoning or output format that is difficult for
the current policy to reproduce. Appendix \ref{sec:prompt_format} shows that the 2B model struggles to optimize under the Mulberry-style \citep{yao2024mulberry} template.

\subsection{Policy Optimization for Language Model Reasoning}
\label{sec:background_rl}

Reinforcement learning (RL) provides an alternative post-training paradigm in
which the model is optimized according to a reward signal rather than fixed
demonstrations. For reasoning tasks, the reward can often be automatically
verified, for example by checking whether the final answer matches the ground
truth. Given a prompt $x$, the policy $\pi_{\theta}$ samples a response $y$, and
a reward function $R(x,y)$ assigns a scalar score. The goal is to maximize
expected reward:
\begin{equation}
J(\theta)
=
\mathbb{E}_{x\sim \mathcal{D},\, y\sim \pi_{\theta}(\cdot \mid x)}
\left[
R(x,y)
\right].
\end{equation}




\textbf{Group Relative Policy Optimization (GRPO)} 
\cite{shao2024deepseekmath} estimates the
baseline using a group of responses sampled for the same prompt. For each prompt
$x$, the old policy samples a group of $G$ responses:
\begin{equation}
\{y_1,\ldots,y_G\}
\sim
\pi_{\theta_{\mathrm{old}}}(\cdot \mid x).
\end{equation}
Each response receives a reward $r_i=R(x,y_i)$. The group-relative advantage is
then computed by normalizing rewards within the group:
\begin{equation}
\hat{A}_i
=
\frac{
r_i-\mathrm{mean}(\{r_j\}_{j=1}^{G})
}{
\mathrm{std}(\{r_j\}_{j=1}^{G})
}.
\end{equation}
Using this relative advantage, GRPO optimizes a PPO-style clipped objective:
\begin{equation}
-\mathbb{E}
\left[
\frac{1}{G}
\sum_{i=1}^{G}
\frac{1}{|y_i|}
\sum_{t=1}^{|y_i|}
\min
\left(
\rho_{i,t}(\theta)\hat{A}_i,\,
\mathrm{clip}(\rho_{i,t}(\theta),1-\epsilon,1+\epsilon)\hat{A}_i
\right)
\right]
+
\beta
D_{\mathrm{KL}}
\left(
\pi_{\theta} \,\|\, \pi_{\mathrm{ref}}
\right),
\end{equation}
where we have defined
\begin{equation}
\rho_{i,t}(\theta)
=
\frac{
\pi_{\theta}(y_{i,t}\mid x,y_{i,<t})
}{
\pi_{\theta_{\mathrm{old}}}(y_{i,t}\mid x,y_{i,<t})
}.
\end{equation}



\textbf{REINFORCE++} emphasizes global advantage normalization rather than prompt-local
normalization \citep{hu2025reinforcepp}. If $\mathcal{B}$ denotes a global batch of sampled responses, the
normalized advantage can be written as
\begin{equation}
\hat{A}(x,y)
=
\frac{
R(x,y)-\mu_{\mathcal{B}}
}{
\sigma_{\mathcal{B}}
},
\end{equation}
where $\mu_{\mathcal{B}}$ and $\sigma_{\mathcal{B}}$ are the mean and standard
deviation of rewards or advantages over the global batch. This design aims to
produce more stable advantage estimates than small prompt-level groups. The resulting objective can be viewed as a critic-free policy-gradient objective:
\begin{equation}
-\mathbb{E}_{(x,y)\sim \mathcal{B}}
\left[
\sum_{t=1}^{|y|}
\hat{A}(x,y)
\log \pi_{\theta}(y_t \mid x,y_{<t})
\right]
+
\beta
D_{\mathrm{KL}}
\left(
\pi_{\theta} \,\|\, \pi_{\mathrm{ref}}
\right).
\end{equation}

We follow recent reasoning-focused RL pipelines that reduce or remove KL ($\beta=0$) both to encourage exploration \citep{hu2025open,entropy2025rethink} and for simplicity.

%% file: sec/4_method.tex
\section{Methodology: Rationale-Guided Policy Optimization}



\subsection{Problem setup}

Let $\mathcal{D}=\{(x_i,s_i,a_i)\}_{i=1}^{M}$ denote the training set, where $x_i$ is the original prompt, $s_i$ is a reference rationale, and $a_i$ is the is the ground-truth answer.  We assume access to a verifiable reward function $r(x,y,a)\in\mathbb{R}$, which assigns reward to a model response $y$ for prompt $x$ given answer $a$.  

We train the model for $N+1$ epochs with the first epoch fully unguided to estimate the model's initial capability on each training example. The following $N$ epochs adaptively reveal or hide increments of size $1/N$ of the reference solution according to whether the current model can already solve the example. For each example, we partition the reference solution into $N$ ordered segments,
\[
s_i = [s_{i,1};s_{i,2};\ldots;s_{i,N}],
\]
where each segment contains approximately $1/N$ of the solution. The partition may be defined by token count, sentence boundaries, reasoning steps, or another deterministic segmentation rule. We employ raw character counts for simplicity in our experiments. Let $g_i^{(t)}\in\{0,1,\ldots,N\}$ denote the number of revealed segments for example $i$ at epoch $t$. The corresponding hint is the prefix
\[
h_i^{(t)} = [s_{i,1};\ldots;s_{i,g_i^{(t)}}],
\]
with $h_i^{(t)}=\emptyset$ when $g_i^{(t)}=0$. We initialize $g_i^{(0)}=0$ for all examples, so the first epoch uses no hint. After each epoch, the hint level is updated adaptively: if the model produces no correct rollout for an example, then the next epoch reveals one additional segment; if the model produces at least one correct rollout, then the next epoch removes one segment. Formally, if $C_i^{(t)}$ indicates whether any rollout for example $i$ is correct at epoch $t$, then
\[
g_i^{(t+1)} =
\begin{cases}
\min(g_i^{(t)}+1,N), & \text{if } C_i^{(t)}=0,\\
\max(g_i^{(t)}-1,0), & \text{if } C_i^{(t)}=1.
\end{cases}
\]
This update rule creates a curriculum that automatically moves each example toward the smallest amount of hinting needed to keep the problem learnable. Hard examples gradually receive more of the reference solution until the model can discover a correct trajectory, while easier examples are annealed back toward the original no-hint setting. In this way, RGPO avoids a fixed hint schedule that may over-assist easy problems or under-assist hard ones.

\subsection{Adaptive rationale-guided rollout}

The first epoch is a standard on-policy RL epoch without any solution hint. For each prompt $x_i$, we sample $K$ rollouts from the current policy,
\[
y_{i,1}^{(0)},\ldots,y_{i,K}^{(0)} \sim \pi_\theta(\cdot\mid x_i),
\]
and compute their rewards $r_{i,k}^{(0)}=r(x_i,y_{i,k}^{(0)},a_i)$. We estimate the initial solvability of each example under the current policy. We define
\[
C_i^{(0)}=\mathbb{I}[\max_k r_{i,k}^{(0)}\geq \tau_{\mathrm{correct}}]
\]
 If $C_i^{(0)}=0$, then none of the sampled unguided responses solved the problem, and RGPO reveals the first $1/N$ of the reference solution for this example in the next epoch. If $C_i^{(0)}=1$, then the example remains unguided in the next epoch because the model has already demonstrated that the problem is learnable without assistance.

For epoch $t\in\{1,\ldots,N\}$, each example is assigned a hint prefix $h_i^{(t)}$ according to the previous epoch's correctness. We then perform two rollout passes. The first pass is the ordinary unguided rollout from the original prompt,
\[
y_{i,k}^{(t)} \sim \pi_\theta(\cdot\mid x_i),
\]
which is used for the main on-policy RL objective. The second pass is a guided revision rollout, used only when a nonempty hint is available or when feedback from the first pass is included. The guided prompt is constructed as
\[
\tilde{x}_i^{(t)} = \mathrm{Reprompt}(x_i,h_i^{(t)},f_i^{(t)}),
\]
where $f_i^{(t)}$ includes verifier feedback and answer-level correctness, ... Note that  $\tilde{x}_i^{(t)}$ is used only to generate candidate improved trajectories.  The guided rollout is sampled as
\[
\tilde{y}_{i,k}^{(t)} \sim \pi_\theta(\cdot\mid \tilde{x}_i^{(t)}),
\]
and evaluated using the same reward function,
\[
\tilde{r}_{i,k}^{(t)} = r(x_i,\tilde{y}_{i,k}^{(t)},a_i).
\]
We compare the guided candidate to the unguided rollout distribution from the same epoch. A guided response is accepted into an auxiliary  buffer only if it improves over the unguided attempts:
\[
\tilde{y}_{i,k}^{(t)} \in \mathcal{B}_{\mathrm{RGPO}}
\quad \Longleftrightarrow \quad
\tilde{r}_{i,k}^{(t)} > \max_{j} r_{i,j}^{(t)} + \delta,
\]
where $\delta\geq 0$ is a margin. In the simplest exact-reward setting, $\delta=0$, so a guided response is accepted whenever it achieves a higher verifier reward than all unguided rollouts for the same example. This reward gate is crucial because the reference solution prefix may not always help the model reason correctly. Instead of assuming that every hint-conditioned response is useful, RGPO treats the guided pass as a proposal distribution and keeps only proposals that improve measured task reward.

\subsection{Policy optimization objective}

RGPO combines a main on-policy RL objective with an auxiliary reward-gated supervised objective. The RL objective is computed on the unguided rollouts sampled from $\pi_\theta(\cdot\mid x_i)$, so the policy is directly optimized for the deployment condition. 
The auxiliary supervised objective is computed on accepted guided responses, but conditioned on the original prompt:
\[
\mathcal{L}_{\mathrm{SFT}}(\theta)
=
-\mathbb{E}_{(x_i,\tilde{y})\in \mathcal{B}_{\mathrm{RGPO}}}
\left[
\log \pi_\theta(\tilde{y}\mid x_i)
\right].
\]
This objective intentionally maximizes $p_\theta(\tilde{y}\mid x_i)$ rather than $p_\theta(\tilde{y}\mid \tilde{x}_i^{(t)})$. The distinction is central to RGPO, because training on $\tilde{x}_i^{(t)}$ would teach the model to rely on rationales at inference time. Training on $x_i$ instead transfers the benefit of the guided revision back into the original task distribution. The full objective for each epoch is
\[
\mathcal{L}_{\mathrm{total}}(\theta)
=
\mathcal{L}_{\mathrm{RL}}(\theta)
+
\lambda_t \mathcal{L}_{\mathrm{SFT}}(\theta)
\]
where $\mathcal{L}_{\mathrm{RL}}$ is the main RLVR objective (e.g. with GRPO or REINFORCE++), and $\lambda_t$ controls the weight of the auxiliary term. Algorithm \ref{alg:rgpo} gives the complete training procedure.
\subsection{Why RGPO helps under sparse rewards}
\label{sec:theory}
Consider binary verifier reward and
let $p \in [0,1]$ be the success probability of one unguided rollout.  RGPO draws $K$
unguided responses and, when guidance is active, $K_g$ guided revisions.  If
all unguided responses fail, let $s$ denote the success probability of one
guided revision.

\begin{proposition}[]
\label{prop:main-recovery}
At equal response budget $K+K_g$, the probabilities of discovering at least
one successful trajectory are
\begin{equation}
P_{\rm RGPO}=1-(1-p)^K(1-s)^{K_g},\qquad
P_{\rm plain}=1-(1-p)^{K+K_g}.
\label{eq:main-recovery}
\end{equation}
Hence $P_{\rm RGPO}>P_{\rm plain}$ iff $s>p$.  In particular, on the
all-failure event---where group-centered binary-reward GRPO has zero
reward-advantage signal---RGPO obtains an accepted positive target with
probability $(1-p)^K[1-(1-s)^{K_g}]$.
\end{proposition}

Hence, at a fixed rollout budget, guidance improves the probability of finding a successful trajectory exactly when the guided success probability exceeds the unguided one.  The scheduler reduces the hint level as the model becomes more likely to solve the problem without guidance. For $0<g_t<N$,
\begin{equation}
\mathbb{E}[g_{t+1}-g_t\mid\mathcal H_t]
=2(1-p_t)^K-1,
\label{eq:main-drift}
\end{equation}
so hints have negative drift once
$p_t>p_\star:=1-2^{-1/K}$ 

\begin{theorem}[]
\label{thm:main-rare-success}
Under the assumptions of the two-response model, suppose sufficiently strong hints
produce a guided success with probability bounded below by a constant $c>0$
while $p_t<\tau$.   Holding the model and hint-quality constants fixed, as
$p_0\downarrow0$,
\begin{equation}
\mathbb{E}[T_\tau^{\rm RGPO}]
=O\!\left(\log\frac1{p_0}\right),\qquad
\mathbb{E}[T_\tau^{\rm plain}]
=\Omega\!\left(\frac1{p_0}\right),
\label{eq:main-rates}
\end{equation}
where plain GRPO uses the same total number of responses per visit.  
\end{theorem}

%% file: sec/5_experiment.tex
\section{Experiments}
\label{sec:experiments}
We begin with a controlled simulation designed to isolate two challenges that motivate RGPO: learning from hard problems for which on-policy sampling rarely discovers a correct solution, and using external rationales that are helpful but not perfectly reliable. We then evaluate RGPO in language-model and vision-language-model post-training to study whether these advantages translate to realistic reasoning tasks. All language- and vision-language-model experiments are conducted using the verl framework \citep{sheng2024hybridflow}.

\subsection{Controlled Simulation}
\label{sec:sim}

\paragraph{Setup.}
Each synthetic problem requires the policy to make a short sequence of reasoning decisions before producing the final answer. Most decision types are \emph{easy}: the initial policy already assigns high probability to a correct choice. A smaller set is \emph{hard}: the initial policy has no preference for the correct choice. Easy problems contain only easy decisions, whereas hard problems contain four difficult decisions. The same decision types recur across different problems, so learning how to handle a difficult decision on one problem can transfer to other problems that require it. Each problem is paired with a reference rationale that can be revealed progressively. To model imperfect supervision, the reference is systematically incorrect for three of the eight hard decision types. Revealing the full rationale also exposes the final answer, creating a potential shortcut. We compare RGPO with unguided GRPO, RL performed in the hinted context, RL with supervised learning on the reference, and always-on full-rationale training, together with ablations of RGPO's reward gate and adaptive hinting. All methods are evaluated on the original problems without hints. 

\begin{figure}[t]
\centering
\includegraphics[width=\columnwidth]{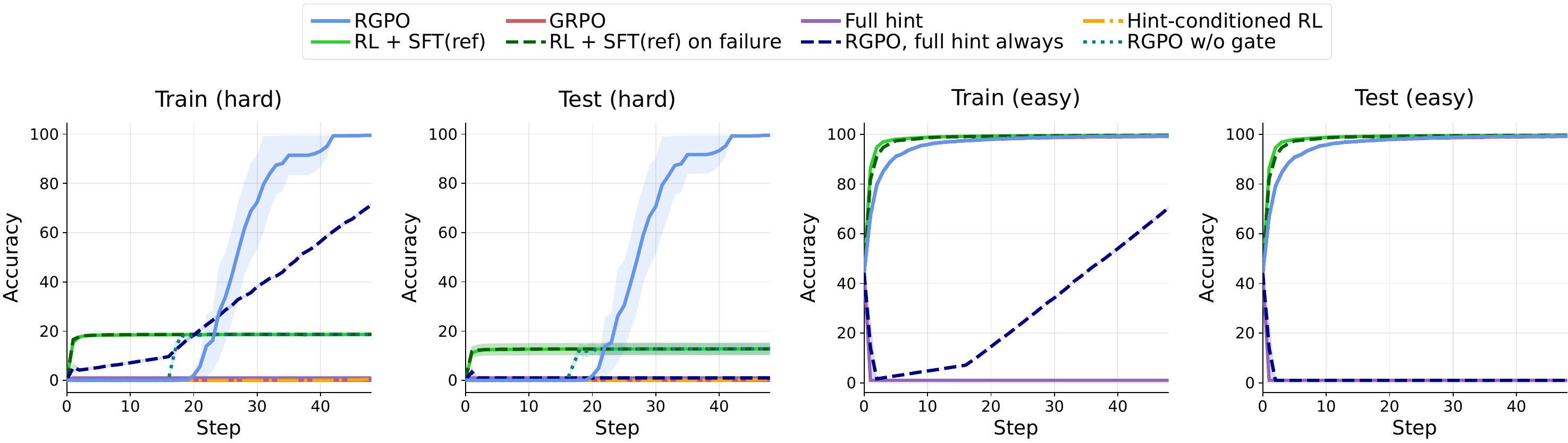}
\caption{\textbf{Controlled simulation.} Accuracy without hints on easy and hard problems, evaluated on both training and held-out problems (mean $\pm$ s.e.\ over 5 seeds). Unguided GRPO readily learns the easy problems but makes no progress on the hard ones. 
}
\label{fig}
\vspace*{-0.4cm}
\end{figure}
Figure~\ref{fig} highlights three aspects of how guidance affects learning. First, \textbf{guidance is
most useful when ordinary on-policy exploration provides little signal}: GRPO reaches $99.1\%$ test accuracy on easy problems but remains at $0.0\%$ on hard problems after three epochs. On
these problems, almost every group of sampled trajectories is incorrect, leaving little useful reward
variation from which to learn. RGPO instead reaches $99.5\%$ hard test accuracy while preserving performance on easy problems. Second, \textbf{more guidance is not always better}. Directly training on the imperfect reference plateaus at $12.8\%$ hard test accuracy because its systematic errors are reinforced. Removing RGPO’s reward gate produces essentially the same behavior, because incorrect guided trajectories are then distilled
together with correct ones. At the opposite extreme, training with the full rationale available allows the model to exploit the exposed final answer, yet performance falls to the $1\%$ chance level when the rationale is removed at evaluation. Third, \textbf{RGPO uses the rationale as a search aid rather than as a training target}. Guided responses are distilled only when they are correct and improve over the unguided samples. In the main setting, only $7.4$ guided responses are accepted per run on average, yet these few verified trajectories are sufficient to bootstrap learning because the same hard decisions recur across problems.  
\subsection{RGPO on Language Models}
For the language-model experiments, we follow the training protocol of HINT~\citep{wang2025hint} and use Qwen2.5-3B~\citep{yang2024qwen25} as the base model. To ensure a consistent comparison across methods, we train on the DAPO-Math-27K split used by HINT, which is derived from the DAPO-Math170K dataset ~\citep{yu2025dapo}. We optimize the model with GRPO ~\citep{shao2024deepseekmath}. Unlike HINT, which trains for 10 epochs, we train for 3 epochs, as we observe that the model already outperforms the compared baselines under this setting. We use a training batch size of 256, a maximum prompt length of 1024 tokens, and a maximum response length of 4096 tokens. For each prompt, the rollout engine samples eight responses, four with hints and four without hints to ensure a fair comparison. The base model is optimized with AdamW~\citep{loshchilov2019decoupled} using a learning rate of $1\times10^{-6}$. For efficient large-model training, we use FSDP2~\citep{zhao2023pytorch}, together with bfloat16 precision~\citep{kalamkar2019study}, fused kernels, and gradient checkpointing~\citep{chen2016training}. We do not enable parameter or optimizer offloading.

\begin{table*}[t]
\caption{\textbf{Overall Performance Comparison of RGPO vs. Baselines.} RGPO obtains the highest in-distribution average, improving over HINT by 5.0 points (26.8 vs. 21.8), and improves the out-of-distribution average by 2.3 points (32.2 vs. 29.9).}
\centering
\small
\resizebox{.95\linewidth}{!}{
\begin{tabular}{l BBBB R ddd R}
\toprule
\multirow{2}{*}{\cellcolor{white}\textbf{Methods}} &
\multicolumn{4}{c}{\cellcolor{white}\textbf{In-Distribution}} &
\multirow{2}{*}{\cellcolor{white}\textbf{Avg}} &
\multicolumn{3}{c}{\cellcolor{white}\textbf{Out-of-Distribution}} &
\multirow{2}{*}{\cellcolor{white}\textbf{Avg}} \\
\cmidrule(lr){2-5} \cmidrule(lr){7-9}
& \cellcolor{white}{AIME} & \cellcolor{white}{Math} & \cellcolor{white}{Olympiad} & \cellcolor{white}{Minerva} & \cellcolor{white}{} & \cellcolor{white}{ARC} & \cellcolor{white}{GPQA-D} & \cellcolor{white}{MMLU-Pro}\\
Vanilla & 2.9 & 39.8 & 12.0 & 9.8 & 16.1 & 44.8 & 11.4 & 28.8 & 28.3 \\
GRPO & 4.3 & 44.0 & 18.2 & 12.2 & 19.7 & 45.0 & 11.8 & 28.0 & 28.3 \\
\midrule
CHORD & 4.5 & 46.6 & 20.2 & 13.0 & 21.1 & 40.0 & 11.0 & 26.4 & 25.8 \\
LUFFY & 3.3 & 40.0 & 18.0 & 13.2 & 18.6 & 40.8 & 11.2 & 24.0 & 25.3 \\
GHPO & 4.0 & 42.2 & 19.6 & 12.8 & 19.7 & 45.5 & \underline{12.0} & 28.2 & 28.6 \\
QuestA & 3.9 & 42.0 & 19.6 & 12.4 & 19.5 & 44.8 & \underline{12.0} & 29.0 & 28.6 \\
BREAD & 4.1 & 44.4 & \underline{20.4} & \underline{13.4} & 20.6 & 45.5 & 11.8 & 29.2 & 28.8 \\
HINT & \underline{4.9} & \underline{48.6} & 20.2 & \underline{13.4} & \underline{21.8} & \underline{48.8} & 11.8 &  \underline{30.2} & \underline{29.9} \\
\midrule
RGPO & \textbf{5.1} & \textbf{55.2} & \textbf{24.9} & \textbf{22.1} & \textbf{26.8} & \textbf{51.2} & \textbf{14.7} &  \textbf{30.8} & \textbf{32.2}\\
\bottomrule
\end{tabular}
}
\label{tab:main_results}
    \vspace*{-0.6cm}
\end{table*}

 We benchmark against the base model, vanilla GRPO, PPO where applicable, and representative approaches for improving rollout quality or efficiency in reinforcement learning with verifiable rewards such as LUFFY \citep{yan2025luffy}, CHORD \citep{zhang2025chord}, GHPO \citep{liu2025ghpo}, QuestA \citep{li2025questa}, BREAD \citep{zhang2025bread}, and HINT \citep{wang2025hint}. Evaluation is conducted without providing hints at test time. We consider both in-distribution and out-of-distribution settings. For mathematical reasoning, we use AIME24, MATH-500 \citep{hendrycks2021math}, OlympiadBench \citep{he2024olympiadbench} and Minerva Math \citep{lewkowycz2022minerva}. Because AIME24 contains a relatively small number of test items, we report avg@32 on this benchmark, while pass@1 is used for the remaining datasets. To examine broader reasoning transfer, we also evaluate on ARC-Challenge \citep{clark2018arc}, GPQA-Diamond \citep{rein2024gpqa}, and MMLU-Pro \citep{wang2024mmlupro}.

Table \ref{tab:main_results} reports the overall performance of RGPO against the vanilla model, standard GRPO, and other hybrid RL training baselines. Overall, RGPO achieves the best result on every in-distribution benchmark, improving the average score from 19.7 with GRPO to 26.8, and outperforming the strongest prior baseline, HINT, by 5.0 points. The gains are especially pronounced on MATH-500 and Minerva Math, where RGPO improves over HINT by 6.6 and 8.7 points, respectively, suggesting that rationale-guided optimization is particularly effective on problems requiring longer or more structured reasoning. RGPO also generalizes well beyond the training distribution, achieving the highest scores on ARC, GPQA, and MMLU, with an out-of-distribution average of 32.2 compared with 29.9 for the best baseline.

\subsection{RGPO on Vision-Language Models}

For vision-language experiments, we follow the experimental protocol of MINT-CoT~\citep{chen2025mintcot} and use Qwen2-VL-7B-Instruct~\citep{wang2024qwen2} as the base multimodal model. We additionally conduct experiments on Qwen2-VL-2B-Instruct as a lower-capacity base model to test whether the benefit persists at smaller scale. Similar to the language model experiment,  we train each model for 3 epochs with adaptive hinting. We use a training batch size of 512, a maximum prompt length of 768 tokens, and a maximum response length of 1024 tokens. 

\paragraph{Evaluation.}
We follow MINT-CoT and evaluate on multimodal mathematical reasoning benchmarks: MathVista-Math \citep{lu2024mathvista} and MMStar-Math \citep{chen2024we}. For reference, we report results from contemporary VLMs, including R1-VL-7B~\citep{chen2025r1v,zhang2025r1vl}, Open-R1-Multimodal~\citep{open-r1-multimodal}, Mulberry~\citep{yao2024mulberry}, MM-Eureka~\citep{meng2025mmeureka}, LLaVA-OneVision-Qwen2-7B-OV~\citep{li2024llava}, InternVL2-8B~\citep{chen2024internvl}, InternVL2-8B-MPO~\citep{wang2024enhancing}, DeepSeek-VL2~\citep{wu2024deepseekvl2}, and Qwen2.5-VL-7B-Instruct~\citep{bai2025qwen25vl}.  We compare RGPO primarily with contemporary VLMs evaluated under comparable conditions. We report later-released state-of-the-art VLMs only as reference points because they often use stronger post-training pipelines. Our goal is to evaluate whether RGPO improves reasoning when applied to a fixed backbone and compared with contemporaneous models and controlled training baselines. 

\begin{figure}[t] 
    \centering
     \includegraphics[width=1.\columnwidth]{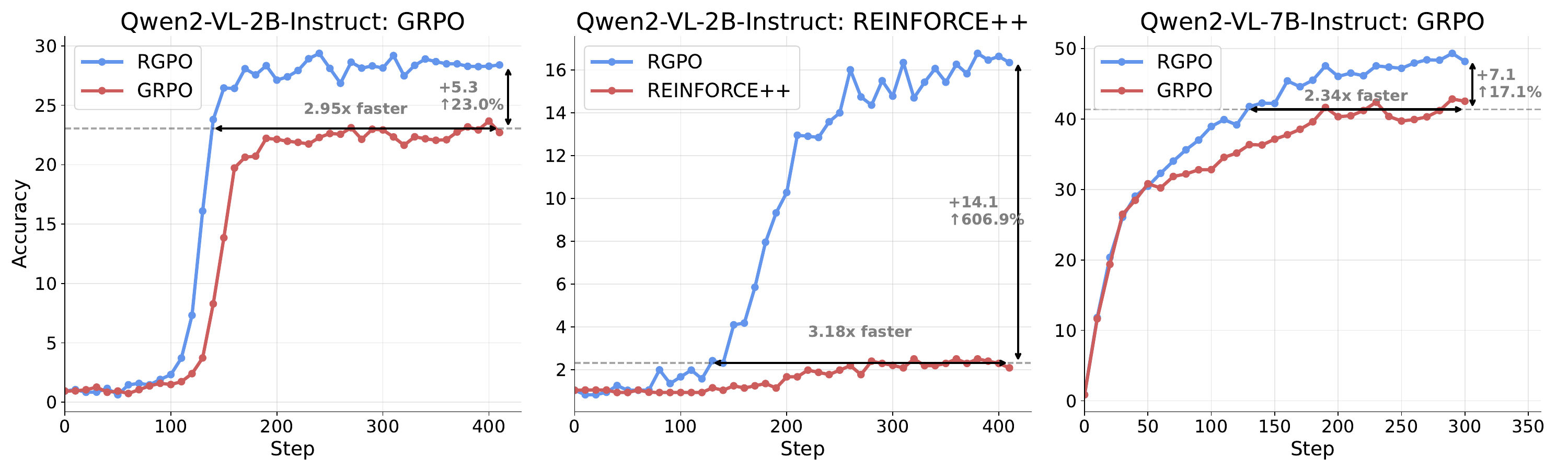}
    \vspace*{-0.5cm}
    \caption{\textbf{RGPO accelerates and improves RLVR training.} Across Qwen2-VL-2B and Qwen2-VL-7B models, RGPO reaches the final performance of RLVR baselines substantially faster and achieves higher final accuracy. The gains are especially pronounced for weaker models and less stable RLVR algorithms (e.g. REINFORCE++).}
    \label{fig:speedup}
    \vspace*{-0.6cm}
\end{figure}

Figure~\ref{fig:speedup} compares RGPO against standard RLVR algorithms, including GRPO and REINFORCE++, on Qwen2-VL-2B-Instruct and Qwen2-VL-7B-Instruct. We report accuracy over training steps to evaluate both final reasoning performance and learning efficiency. Across all settings, RGPO improves the data efficiency over the corresponding RLVR baseline. On Qwen2-VL-2B with GRPO, RGPO reaches the baseline performance in 2.95× fewer training steps and achieves a higher final accuracy. The improvement is even larger when compared with REINFORCE++, where the baseline struggles to learn effectively, while RGPO continues to improve and obtains a substantial final gain. On the stronger Qwen2-VL-7B model, both methods improve during training, but RGPO still converges more efficiently and reaches a higher final accuracy. 

\begin{table}[!ht]
  \caption{\textbf{Performance comparison on MMStar-Math and the mathematical subset of MathVista} We compare RGPO with the base Qwen2-VL-7B-Instruct model, CoT supervised fine-tuning variants, and representative open-source multimodal reasoning models.  Dashes indicate unavailable results.}
  \label{tab:mathvista}
  \centering
  \resizebox{.95\linewidth}{!}{%
  \begin{tabular}{l|c|ccccc}
    \toprule
\multirow{2}{*}{Model} & \multirow{2}{*}{MMStar-Math} & \multicolumn{5}{c}{MathVista-Math}\\
\cmidrule{3-7}
 & & All & GEO & ALG & GPS & TQA \\
\midrule
LLaVA-OneVision-Qwen2-7b-ov~\cite{li2024llava} & -- & 67.04 & 69.34 & 67.04 & 69.71 & 58.06\\
InternVL2-8B~\cite{chen2024internvl} & 66.8 & 62.59 & 62.26 & 62.92 & 62.50 & 62.90\\
InternVL2-8B-MPO~\cite{wang2024enhancing} & -- & 68.52 & 68.87 & 68.91 & 69.71 & 64.52\\
DeepSeek-VL2~\cite{wu2024deepseekvl2} & -- & 65.56 & 63.68 & 65.54 & 63.94 & 70.97\\
Qwen2.5-VL-7B-Instruct~\cite{bai2025qwen25vl} & 66.8 & 66.66 & 65.56 & 66.29 & 65.87 & 69.35\\
Open-R1-Multimodal~\cite{open-r1-multimodal} & 59.2 & 54.81 & 52.36 & 54.68 & 53.37 & 59.68\\
R1-VL-7B~\cite{zhang2025r1vl} & 68.4 & 69.63 & 68.87 & 69.66 & 69.71 & 69.35 \\
Mulberry~\cite{yao2024mulberry} & 66.8 & 68.52 & 67.92 & 68.54 & 68.75 & 67.74 \\
MM-Eureka~\cite{meng2025mmeureka} & -- & \underline{72.59} & \underline{71.22} & \underline{72.66} & \underline{72.60} & \textbf{72.58} \\
\midrule
Qwen2-VL-7B-Instruct~\cite{wang2024qwen2vlenhancingvisionlanguagemodels} & 46.4 & 41.11 & 35.85 & 41.57 & 36.54 & 56.45\\
Original Image CoT SFT & --  & 40.37 &  {38.68} & 40.82 & {39.42} & 43.54  \\
Bounding Box CoT SFT & --  &  65.56 & 63.21 & 65.54 & 63.94 & \underline{70.97} \\
 Text-only CoT SFT &  \underline{67.6} & 64.07  &  64.15   & 64.04  & 64.42  &  62.90 \\
 RGPO & \textbf{70.1} & \textbf{73.65} & \textbf{73.47} & \textbf{73.62} & \textbf{73.89} & {70.35} \\
\bottomrule

\end{tabular}
}
\end{table}

As shown in Table \ref{tab:mathvista}, RGPO delivers the strongest overall performance among the compared methods. On MMStar-Math, RGPO reaches 70.1, outperforming the base Qwen2-VL-7B-Instruct model by 23.7 points and surpassing other open-source multimodal reasoning models such as R1-VL-7B and Mulberry. On MathVista-Math, RGPO achieves the highest overall score of 73.65, improving substantially over the base model’s 41.11 and also outperforming strong reasoning-oriented baselines such as MM-Eureka. The gains are especially clear in the GEO, ALG, and GPS subsets, where RGPO obtains the best scores of 73.47, 73.62, and 73.89, respectively. In contrast, the gain on the Textbook Question Answering subset is relatively modest. This may be due to a small degree of domain overlap between the MINT-CoT~\citep{chen2025mintcot} training set and this subset of MathVista, which leaves less room for additional improvement on this category.

\begin{figure}[t] 
    \centering
    \vspace*{-.4cm}
     \includegraphics[width=1.\columnwidth]{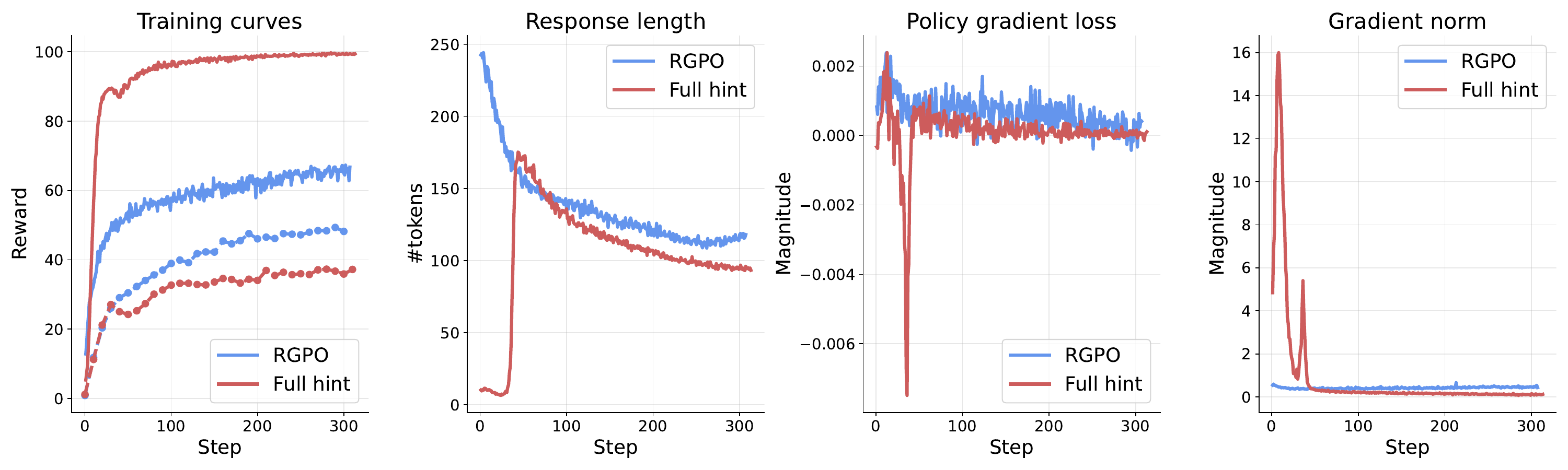}
    \caption{\textbf{ Effect of failure-conditioned hinting versus always-on full hints:} Training (solid lines) and validation (dashed lines) rewards. Full hints rapidly inflate training reward but generalize poorly, while RGPO learns more gradually and achieves stronger no-hint validation performance. }
    \label{fig:full_hint_vs_rgpo}
    \vspace*{-.8cm}
\end{figure}

Figure \ref{fig:full_hint_vs_rgpo} compares RGPO with a full-hint setting, where the model is given the complete hint during training. Although the full-hint model quickly obtains high training reward, this improvement is misleading: the model can simply use the provided hint to output the correct answer with only a few tokens, without learning to perform the reasoning process independently. This behavior indicates shortcut behavior, where the policy exploits the availability of the hint rather than developing robust reasoning ability. The full-hint model initially produces extremely short outputs while still receiving high reward, indicating that it can bypass reasoning and directly answer from the hint. Moreover, the policy-gradient loss and gradient-norm curves show that full-hint training is less stable, with large early spikes and sharp fluctuations. In contrast, RGPO achieves more gradual reward improvement, maintains longer reasoning responses, and exhibits smoother optimization dynamics. These results suggest that always providing hints can create a shortcut that produces less stable optimization and poorer no-hint validation performance, while RGPO’s failure-conditioned hinting mechanism provides guidance only when needed and avoids encouraging the model to rely on hints as a substitute for reasoning.

\begin{wrapfigure}{r}{0.5\columnwidth}
    \vspace*{-2.2em}
    \centering
\includegraphics[width=0.5\columnwidth]{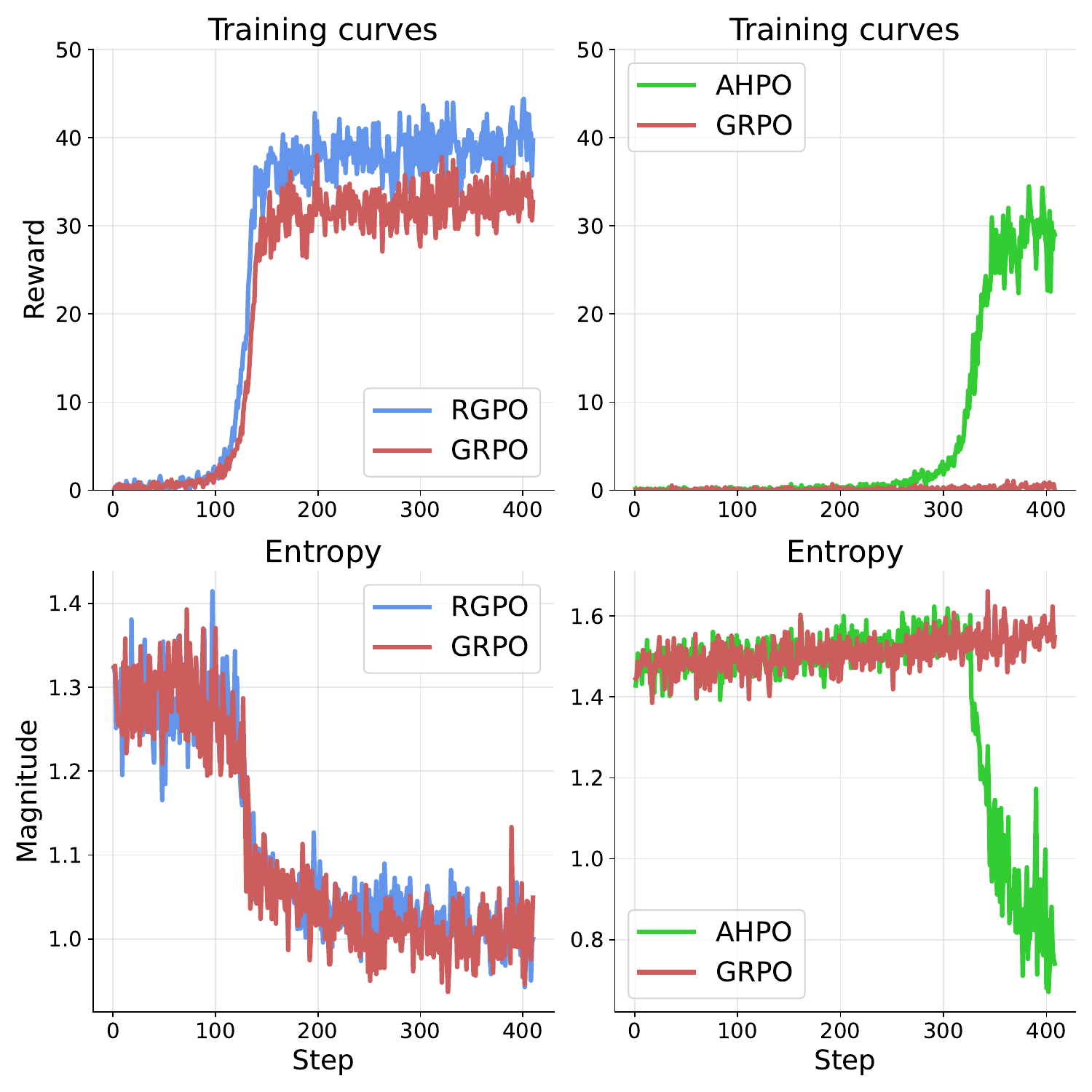}
    \caption{\textbf{Training dynamics under different reasoning prompt formats.} Left: using the concise DeepSeek-R1-style format, RGPO achieves higher reward than GRPO while maintaining a comparable entropy trajectory throughout training. Right: under the original Mulberry-style template, AHPO eventually improves reward after learning the required output structure, but this improvement is accompanied by a sharp entropy collapse, indicating reduced exploration and more deterministic generation behavior.}
    \label{fig:srft_vs_rgpo}
    \vspace{-1em}
\end{wrapfigure}

Figure~\ref{fig:srft_vs_rgpo} highlights how RGPO can yield a strong reasoning model even when the rationale format is poorly matched to the policy. Unlike methods that directly optimize likelihood on the reference rationale, RGPO does not require the policy to reproduce the rationale format and   is therefore compatible with different response templates. With the concise DeepSeek-R1-style prompt, both RGPO and GRPO begin improving rapidly once the model discovers the correct answer format, but RGPO reaches a consistently higher reward plateau. At the same time, its entropy decreases smoothly and remains close to GRPO, suggesting that the performance gain does not come from simply collapsing the policy into a narrow set of outputs. In contrast, when using the more structured Mulberry-style prompt, AHPO \citep{zhao2025mm}, which combines off-policy expert SFT loss with RLVR, shows a delayed reward increase but undergoes a sharp entropy collapse shortly afterward. This suggests that AHPO learns to satisfy the complex template by concentrating probability mass on a limited response pattern, reducing exploration once the formatting objective is acquired. We therefore use the simpler DeepSeek-R1-style format in our main experiments: it makes the reward signal easier to optimize while allowing RGPO to improve reasoning performance without prematurely sacrificing policy diversity.

%% file: sec/6_conclusion.tex
\section{Conclusion}

In this paper, we introduced Rationale-Guided Policy Optimization, an RL framework that improves reasoning by using rationales as temporary scaffolds rather than direct imitation targets. RGPO's design reduces hint dependence by preventing the model from relying on rationales for every example, and reduces rationale memorization by avoiding direct training on the rationale text itself. Motivated by cognitive principles, RGPO treats rationales as instructional interventions that help convert failed attempts into useful learning signals. Empirically, RGPO consistently shows superior performance on post-training of both language and vision-language models.

%% file: sec/7_supp.tex
Due to space constraints, some implementation details, additional analyses, and supplementary results are omitted from the main paper. In this appendix, we provide a more complete description of the RGPO training procedure, including the prompt templates, algorithmic details, and hyperparameter settings. We also include additional experimental results and ablation studies to better understand the effects of rationale guidance, prompt format, and failure-conditioned hinting on model learning and final performance.

\section{Method Details}
\label{app:method}

This section provides implementation details for Rationale-Guided Policy Optimization (RGPO) that are omitted from the main text. We first give the complete training procedure and then clarify several design choices that are important in practice.

\subsection{Training Procedure}
\label{app:algorithm}
Algorithm~\ref{alg:rgpo} summarizes Rationale-Guided Policy Optimization. In our experiments, we set the correctness threshold to $\tau_{\mathrm{correct}}=1.0$, requiring an unguided rollout to be perfectly correct in order to be considered solved. We set the acceptance margin to $\delta=0$, so a rationale-guided revision is added to the auxiliary buffer whenever it achieves a strictly higher verifier reward than the best unguided rollout. The rationale-guided loss coefficient $\lambda_t$ is tuned over ${10^{-2}, 10^{-3}, 10^{-4}}$.

\begin{algorithm}[!ht]
\caption{Rationale-Guided Policy Optimization}
\label{alg:rgpo}
\begin{algorithmic}[1]
\Require Training set $\mathcal{D}=\{(x_i,s_i,a_i)\}_{i=1}^{M}$, policy $\pi_\theta$, verifier reward $r$, number of hint segments $N$, rollouts per prompt $K$, acceptance margin $\delta$
\State Partition each solution $s_i$ into $N$ segments $[s_{i,1};\ldots;s_{i,N}]$
\State Initialize hint level $g_i^{(0)}\gets 0$ for all examples
\For{$t=0$ to $N$}
    \State Clear auxiliary buffer $\mathcal{B}_{\mathrm{RGPO}}$
    \For{each minibatch of examples}
        \State Sample unguided rollouts $y_{i,1}^{(t)},\ldots,y_{i,K}^{(t)}\sim\pi_\theta(\cdot\mid x_i)$
        \State Compute rewards $r_{i,k}^{(t)}=r(x_i,y_{i,k}^{(t)},a_i)$ and advantages for the main RL objective
        \If{$t>0$ and $g_i^{(t)}>0$}
            \State Construct hint prefix $h_i^{(t)}=[s_{i,1};\ldots;s_{i,g_i^{(t)}}]$
            \State Construct guided prompt $\tilde{x}_i^{(t)}=\mathrm{Reprompt}(x_i,h_i^{(t)},f_i^{(t)})$
            \State Sample guided revisions $\tilde{y}_{i,1}^{(t)},\ldots,\tilde{y}_{i,K}^{(t)}\sim\pi_\theta(\cdot\mid \tilde{x}_i^{(t)})$
            \State Compute guided rewards $\tilde{r}_{i,k}^{(t)}=r(x_i,\tilde{y}_{i,k}^{(t)},a_i)$
            \State Add $(x_i,\tilde{y}_{i,k}^{(t)})$ to $\mathcal{B}_{\mathrm{RGPO}}$ if $\tilde{r}_{i,k}^{(t)}>\max_j r_{i,j}^{(t)}+\delta$
        \EndIf
        \State Update $\pi_\theta$ using $\mathcal{L}_{\mathrm{RL}}+\lambda_t\mathcal{L}_{\mathrm{SFT}}$
    \EndFor
    \For{each example $i$}
        \State Set $C_i^{(t)}=\mathbb{I}[\max_k r_{i,k}^{(t)}\geq\tau_{\mathrm{correct}}]$ using unguided rollouts
        \If{$C_i^{(t)}=0$}
            \State $g_i^{(t+1)}\gets \min(g_i^{(t)}+1,N)$
        \Else
            \State $g_i^{(t+1)}\gets \max(g_i^{(t)}-1,0)$
        \EndIf
    \EndFor
\EndFor
\end{algorithmic}
\end{algorithm}

\textbf{On-Policy Self-Distillation (OPSD)} Compared to RGPO, exact on-policy distillation \citep{agarwal2024policy,zhao2026self,shenfeld2026self} would require access to the teacher's token-level probabilities or logits for the actions taken under the student's rollout distribution. Even when the teacher is open, computing this KL term introduces substantial memory and compute overhead because training must run an additional teacher forward pass and retain token-level probability information, often over large vocabularies and long reasoning trajectories. Treating teacher-sampled completions as ordinary SFT targets avoids the logit requirement, but becomes only an off-policy approximation and can reintroduce the distribution-mismatch and format-imitation issues discussed in the main paper. RGPO avoids these dependencies by using external rationales only as prompt-level scaffolds: the refined candidate is generated by the trainable policy itself, accepted only when it improves verifier reward, and then distilled back under the original no-hint prompt. Consequently, RGPO can exploit black-box, human-written, or dataset-provided rationales without requiring teacher logits or an additional teacher model during training, while still optimizing the deployed policy for the original inference condition.

Below, we provide three qualitative examples illustrating how \textcolor{partialsolution}{partial rationales}, together with the model's \textcolor{judgefeedback}{verifier-judged} \textcolor{previousresponse}{response}, can guide the model toward the \textcolor{refinedanswer}{correct solution}. In these examples, attempting to answer the question directly can lead to an \textcolor{previousresponse}{incorrect answer}; after receiving auxiliary information, the model is steered toward a {\textcolor{refinedanswer}{correct refined response}}. The model is then trained to reproduce such \textcolor{refinedanswer}{corrected response} without relying on the auxiliary context. 

\begin{Example}{\label{ex:1} The question for generating the refined answer consists of the following sections, indicated by their corresponding colors: \textcolor{originalquestion}{original question}, \textcolor{partialsolution}{partial solution}, \textcolor{previousresponse}{previous response}, \textcolor{judgefeedback}{previous judge feedback}, and \textcolor{instruction}{further instructions}.}{}
\textbf{System prompt:}  You are a helpful AI assistant. A conversation takes place between the User and the Assistant. The User asks a question, and the Assistant solves it. Please help me solve this question. Wrap only the final answer in \texttt{\textbackslash boxed\{\}
}.

\textbf{Question:} {\color{originalquestion}{%
Human beings have discovered a habitable planet and soon after, they find 10 more habitable planets. Of these 11, only 5 are deemed ``Earth-like'' in their resources and the rest are deemed ``Mars-like'' since they lack many important resources. Assume that planets like Earth take up 2 units of colonization, while those like Mars take up only 1. If humanity mobilizes 12 total units of colonies, how many different combinations of planets can be occupied if the planets are all different from each other?}}

{\color{partialsolution}{%
This is the (partial) solution without formatting, you can leverage it to check your answer:
Think about combinations and constraints. How might you use the fact that Earth-like planets take up 2 units and Mars-like planets take up 1 unit to de
}}

{\color{previousresponse}
Previous response:
To solve this problem, we need to consider the constraints and find combinations that meet them. Here’s a step-by-step breakdown:

1. **Earth-like Planets**: Each Earth-like planet requires 2 units of colonization.
2. **Mars-like Planets**: Each Mars-like planet requires 1 unit of colonization.
3. **Total Units of Colonization**: There are 12 units available.

We need to find combinations of Earth-like and Mars-like planets such that their total units add up to 12. Let's denote the number of Earth-like planets by \( e \) and the number of Mars-like planets by \( m \). We then have:
\[ 2e + m = 12 \]

We also need to consider that the planets are distinct, so we need to count the number of ways to choose these combinations. Specifically, we need to count the valid pairs \((e, m)\) where both \( e \) and \( m \) are integers between 0 and a maximum that keeps the total units within the limit.

Let’s solve for different values of \( m \):

* If \( m = 0 \):
  \[ 2e = 12 \Rightarrow e = 6 \] (This means 6 Earth-like planets and 0 Mars-like planets). The number of ways to choose 6 Earth-like planets out of 5 is:
  \[ \binom{5}{6} = 0 \]

* If \( m = 1 \):
  \[ 2e + 1 = 12 \Rightarrow 2e = 11 \Rightarrow e = 5.5 \] (This is not an integer, so not feasible).

* If \( m = 2 \):
  \[ 2e + 2 = 12 \Rightarrow 2e = 10 \Rightarrow e = 5 \] {\color{previousresponse}(This means 5 Earth-like planets and 2 Mars-like planets). The number of ways to choose 5 Earth-like planets out of 5 is:
  \[ \binom{5}{5} \cdot \binom{6}{2} = 1 \cdot \binom{6}{2} = 1 \cdot 15 = 15 \]

* If \( m = 3 \):
  \[ 2e + 3 = 12 \Rightarrow 2e = 9 \Rightarrow e = 4.5 \] (This is not an integer, so not feasible).

* If \( m = 4 \):
  \[ 2e + 4 = 12 \Rightarrow 2e = 8 \Rightarrow e = 4 \] (This means 4 Earth-like planets and 4 Mars-like planets). The number of ways to choose 4 Earth-like planets out of 5 is:
  \[ \binom{5}{4} \cdot \binom{6}{4} = 5 \cdot 15 = 75 \]

* If \( m = 5 \):
  \[ 2e + 5 = 12 \Rightarrow 2e = 7 \Rightarrow e = 3.5 \] (This is not an integer, so not feasible).

* If \( m = 6 \):
  \[ 2e + 6 = 12 \Rightarrow 2e = 6 \Rightarrow e = 3 \] (This means 3 Earth-like planets and 6 Mars-like planets). The number of ways to choose 3 Earth-like planets out of 5 is:
  \[ \binom{5}{3} \cdot \binom{6}{6} = 10 \cdot 1 = 10 \]

* If \( m = 7 \):
  \[ 2e + 7 = 12 \Rightarrow 2e = 5 \Rightarrow e = 2.5 \] (This is not an integer, so not feasible).

* If \( m = 8 \):
  \[ 2e + 8 = 12 \Rightarrow 2e = 4 \Rightarrow e = 2 \] (This means 2 Earth-like planets and 8 Mars-like planets). The number of ways to choose 2 Earth-like planets out of 5 is:
  \[ \binom{5}{2} \cdot \binom{6}{8} = 10 \cdot 1 = 10 \]

* If \( m = 9 \):
  \[ 2e + 9 = 12 \Rightarrow 2e = 3 \Rightarrow e = 1.5 \] (This is not an integer, so not feasible).

* If \( m = 10 \):
  \[ 2e + 10 = 12 \Rightarrow 2e = 2 \Rightarrow e = 1 \] (This means 1 Earth-like planet and 10 Mars-like planets). The number of ways to choose 1 Earth-like planet out of 5 is:
  \[ \binom{5}{1} \cdot \binom{6}{10} = 5 \cdot 0 = 0 \]

* If \( m = 11 \):
  \[ 2e + 11 = 12 \Rightarrow 2e = 1 \Rightarrow e = 0.5 \] (This is not an integer, so not feasible).

Adding up all feasible cases, we get:
\[ 0 + 15 + 75 + 10 + 10 + 0 = 110 \]
{\color{previousresponse}
Thus, the number of different combinations of planets that can be occupied is \texttt{\textbackslash boxed\{110\}}.
}}}

{\color{judgefeedback}{
Feedback for your previous response:

Format score: 1.0/1.0

Accuracy score: 0/1.0

Overall score: 0.1/1.0

}}

{\color{instruction}{%
You should maximize the overall score by improving both the format and accuracy of your answer.

Write a new standalone solution to the original question without mentioning the previous responses or feedback. You first think through the reasoning process as an internal monologue. Then, provide your final answer enclosed within \texttt{\textbackslash boxed\{\}}. Please strictly follow the format.
}}

\textbf{Answer:} {\color{refinedanswer}{We are trying to find different combinations of Earth-like (2 units) and Mars-like (1 unit) planets. There are 12 units of colonization available, and we have 5 Earth-like and 6 Mars-like planets.

Here’s a step-by-step approach:

1. **Define Variables**: Let \( e \) be the number of Earth-like planets and \( m \) be the number of Mars-like planets. The total units of colonization constraint can be written as \( 2e + m = 12 \).

2. **Constraints on \( e \) and \( m \)**: Since we have 5 Earth-like planets and 6 Mars-like planets, the values of \( e \) and \( m \) must satisfy \( 0 \leq e \leq 5 \) and \( 0 \leq m \leq 6 \).

3. **Possible Solutions**: We need to find all possible valid pairs \((e, m)\) that satisfy the equation \( 2e + m = 12 \) within the given constraints.

Let’s enumerate these pairs:

- For \( e = 0 \): \( 2(0) + m = 12 \Rightarrow m = 12 \) (not possible since \( m \leq 6 \)).

- For \( e = 1 \): \( 2(1) + m = 12 \Rightarrow m = 10 \) (not possible since \( m \leq 6 \)).

- For \( e = 2 \): \( 2(2) + m = 12 \Rightarrow m = 8 \) (not possible since \( m \leq 6 \)).

- For \( e = 3 \): \( 2(3) + m = 12 \Rightarrow m = 6 \) (possible).

- For \( e = 4 \): \( 2(4) + m = 12 \Rightarrow m = 4 \) (possible).

- For \( e = 5 \): \( 2(5) + m = 12 \Rightarrow m = 2 \) (possible).

From these valid pairs, we count the number of ways to choose the required number of Earth-like and Mars-like planets.

4. **Calculate Combinations**:

   - For \( (3, 6) \): We need to choose 3 out of 5 Earth-like planets and 6 out of 6 Mars-like planets. The number of ways is \( \binom{5}{3} \times \binom{6}{6} = 10 \times 1 = 10 \).
   
   - For \( (4, 4) \): We need to choose 4 out of 5 Earth-like planets and 4 out of 6 Mars-like planets. The number of ways is \( \binom{5}{4} \times \binom{6}{4} = 5 \times 15 = 75 \).
   
   - For \( (5, 2) \): We need to choose 5 out of 5 Earth-like planets and 2 out of 6 Mars-like planets. The number of ways is \( \binom{5}{5} \times \binom{6}{2} = 1 \times 15 = 15 \).

Adding these together gives us the total number of combinations:
\( 10 + 75 + 15 = 100 \).

Therefore, the number of different combinations of planets that can be occupied is \texttt{\textbackslash boxed\{100\}}.}}

\textbf{Judge:}

Solution: 100

Ground truth: 100

Format score: 1.0

Accuracy score: 1.0

Overall score: 1.0

\end{Example}

In Example \ref{ex:1}, the model leverages the partial solution to identify the correct counting strategy and formulate the constraint (2e+m=12). It then uses the previous response to retain the valid setup while correcting its main error: allowing more than the six available Mars-like planets. By enforcing the bounds $(e\leq5)$ and $(m\leq6)$, the model considers only the valid cases ((3,6)), ((4,4)), and ((5,2)), counts the distinct planet selections with combinations, and obtains the correct total of (100).

\begin{Example}{\label{ex:2}The question for generating the refined answer consists of the following sections, indicated by their corresponding colors: \textcolor{originalquestion}{original question}, \textcolor{partialsolution}{partial solution}, \textcolor{previousresponse}{previous response}, \textcolor{judgefeedback}{previous judge feedback}, and \textcolor{instruction}{further instructions}.}{}
\textbf{System prompt:}  You are a helpful AI assistant. A conversation takes place between the User and the Assistant. The User asks a question, and the Assistant solves it. Please help me solve this question. Wrap only the final answer in \texttt{\textbackslash boxed\{\}
}.

\textbf{Question:} {\color{originalquestion}{%
Each corner of a rectangular prism is cut off. Two (of the eight) cuts are shown. How many edges does the new figure have?
}}

{\color{partialsolution}{%
This is the (partial) solution without formatting, you can leverage it to check your answer:
Think about the original number of edges on a rectangular prism. How does cutting of
}}

{\color{previousresponse}{
Previous response: acing its corners, we need to determine how many vertices it has.
Since each cut off a vertex adds one new edge and removes two, the net increase in edges is \( 2 \times 4 = 8 \).
So, the number of edges is \texttt{\textbackslash boxed\{24\}}.
}}

{\color{judgefeedback}{
Feedback for your previous response:

Format score: 1.0/1.0

Accuracy score: 0/1.0

Overall score: 0.1/1.0

}}

{\color{instruction}{%
You should maximize the overall score by improving both the format and accuracy of your answer.

Write a new standalone solution to the original question without mentioning the previous responses or feedback. You first think through the reasoning process as an internal monologue. Then, provide your final answer enclosed within \texttt{\textbackslash boxed\{\}}. Please strictly follow the format.
}}

\textbf{Answer:} {\color{refinedanswer}{To determine the number of edges in the new figure after cutting off the corners of a rectangular prism, we'll first understand the initial properties and then consider the effects of cutting off the corners.

\#\#\# Initial Number of Edges

A rectangular prism has 8 vertices, 12 edges, and 6 faces. However, since we are asked to count the edges created by cutting off the corners, it's more efficient to consider the total vertices and faces first.

\#\#\# Basic Structure Before Cutting

1. **Vertices**: 8
2. **Edges**: 12
3. **Faces**: 6 (2 pairs of identical faces)

Now, let's think about the effect of cutting off corners:

- Each corner cut off removes one vertex and creates 3 new edges (one for each edge of the faces that used to meet at the vertex).

\#\#\# Analysis After Cutting Off Corners
When we cut off each corner, we effectively:
- Remove 1 vertex from the original 8 vertices.
- Add 3 new edges.

Since we are cutting off all 8 corners of the rectangular prism:
- **Total number of edges added by cutting off corners**: \( 8 \times 3 = 24 \).

\#\#\# Calculation of Final Edge Count
Initially, the rectangular prism has 12 edges. After cutting off the corners, these initial 12 edges plus the additional 24 edges give us:

\[
12 + 24 = 36
\]

Thus, the new figure has \texttt{\textbackslash boxed\{36\}}.}}

\textbf{Judge:}

Solution: 36

Ground truth: 36

Format score: 1.0

Accuracy score: 1.0

Overall score: 1.0

\end{Example}
Leveraging the previous response and the partial solution in Example \ref{ex:2} helped the model identify the correct structure of the problem: begin with the rectangular prism’s 12 original edges and then account for the effect of cutting off all 8 corners. The earlier incorrect response highlighted that the edge change needed to be reconsidered, while the partial solution directed attention to how each cut modifies the prism. Using these cues, the model correctly determined that each corner cut creates 3 new edges, resulting in $(12+8\cdot3=36)$ edges.

\begin{Example}{\label{ex:3}The question for generating the refined answer consists of the following sections, indicated by their corresponding colors: \textcolor{originalquestion}{original question}, \textcolor{partialsolution}{partial solution}, \textcolor{previousresponse}{previous response}, \textcolor{judgefeedback}{previous judge feedback}, and \textcolor{instruction}{further instructions}.}{}
\textbf{System prompt:}  You are a helpful AI assistant. A conversation takes place between the User and the Assistant. The User asks a question, and the Assistant solves it. Please help me solve this question. Wrap only the final answer in \texttt{\textbackslash boxed\{\}
}.

\textbf{Question:} {\color{originalquestion}{%
Each corner of a rectangular prism is cut off. Two (of the eight) cuts are shown. How many edges does the new figure have?
}}

{\color{partialsolution}{%
This is the (partial) solution without formatting, you can leverage it to check your answer:
Think about the original number of edges on a rectangular prism. How does cutting of
}}

{\color{previousresponse}{
To find the smallest perfect square larger than 1 with a perfect square number of positive integer factors, we can consider perfect squares of factors of perfect numbers. Here's an explanation:

A perfect square number of positive integer factors means that the number of factors of the perfect square is itself a perfect square.

We can start by considering the perfect numbers, which are integers equal to the sum of their proper positive factors. Examples include 6 (1 + 2 + 3) and 28 (1 + 2 + 4 + 7 + 14).

The smallest relative prime pair (ignoring multiples of 2) that gives a perfect number is (4, 9). When we multiply these, we get a perfect square, 36, which has (1, 4, 9, and 36) perfect factors, all of which are perfect squares.

Therefore, the smallest perfect square larger than 1 with a perfect square number of positive integer factors is 36. 
}}

{\color{judgefeedback}{
Feedback for your previous response:

Format score: 0.0/1.0

Accuracy score: 0/1.0

Overall score: 0.0/1.0

}}

{\color{instruction}{%
You should maximize the overall score by improving both the format and accuracy of your answer.

Write a new standalone solution to the original question without mentioning the previous responses or feedback. You first think through the reasoning process as an internal monologue. Then, provide your final answer enclosed within \texttt{\textbackslash boxed\{\}}. Please strictly follow the format.
}}

\textbf{Answer:} {\color{refinedanswer}{To find the smallest perfect square larger than 1 with a perfect square number of positive integer factors, we can start by considering the prime factorization of perfect squares. A perfect square can be expressed as the product of pairs of prime factors. If we denote a perfect square as \( n = p_1^{2a_1} \cdot p_2^{2a_2} \cdots p_k^{2a_k} \), where \( p_i \) are distinct prime numbers and \( a_i \) are integers, the number of factors is given by \( (2a_1 + 1)(2a_2 + 1) \cdots (2a_k + 1) \).

For the number of factors to be a perfect square, each term in the product \( (2a_1 + 1)(2a_2 + 1) \cdots (2a_k + 1) \) must be a perfect square or a product of factors that can be paired to form a perfect square. 

Let's consider the smallest perfect square that appears in the set of perfect numbers. The smallest perfect number is 6, whose prime factorization is \( 2 \cdot 3 \). When we consider the smallest relative prime pair (ignoring multiples of 2), which is (4, 9), their product is \( 36 \). Here, 36 has the prime factorization \( 2^2 \cdot 3^2 \), which has \( (2a_1 + 1)(2a_2 + 1) = 3 \cdot 3 = 9 \), a perfect square. Therefore, 36 is a perfect square with a perfect square number of positive integer factors.

Thus, the smallest perfect square larger than 1 with a perfect square number of positive integer factors is \texttt{\textbackslash boxed\{36\}}.}}

\textbf{Judge:}

Solution: 36

Ground truth: 36

Format score: 1.0

Accuracy score: 1.0

Overall score: 1.0

\end{Example}

In Example \ref{ex:3}, the model uses the judge’s feedback to correct the formatting by placing only the final answer inside \texttt{\textbackslash boxed\{\}} and to ensure that the answer is parsed properly. Although some unnecessary reasoning from the earlier response remains, these changes allow the refined response to receive full accuracy and format scores.

\section{Theoretical Analysis of Rationale-Guided Learning}
\label{sec:rgpo_theory}


For a prompt $x$, let $R(x,y)$ denote the verifier reward and
$\pi_\theta(y\mid x)$ the policy likelihood of a complete response.
We distinguish $\pi_\theta$ from the decoding distribution
$\rho_\theta(\cdot\mid x)$ used to generate rollouts. These coincide under
ancestral sampling, but need not coincide under temperature scaling or
truncation. Sampling and scheduler quantities are defined with respect to
$\rho_\theta$, whereas likelihood-based quantities are defined with respect to
$\pi_\theta$.

At a fixed policy $\theta$, sample $K$ unguided responses independently as
\[
Y_1,\ldots,Y_K \sim \rho_\theta(\cdot\mid x),
\qquad
B=(Y_1,\ldots,Y_K).
\]
When guidance is applied, $K_g$ additional candidates
$\widetilde Y_1,\ldots,\widetilde Y_{K_g}$ are generated conditional on the
unguided group $B$, the current rationale prefix, and verifier feedback.
A candidate is accepted if
\begin{equation}
R(x,\widetilde Y_m)
>
\max_{1\leq j\leq K} R(x,Y_j)+\delta,
\qquad \delta\geq 0.
\label{eq:app-gate}
\end{equation}
Let $A$ denote the event that at least one guided candidate is accepted, and
let $N_{\rm acc}$ denote the number of accepted candidates. For binary rewards, $R\in\{0,1\}$, define the unguided success probability
\[
p
=
\Pr_{Y\sim\rho_\theta(\cdot\mid x)}
\bigl(R(x,Y)=1\bigr),
\]
and the all-failure event
$
F
=
\{R(x,Y_j)=0,\ \forall j=1,\ldots,K\}.
$. When $\rho_\theta=\pi_\theta$, we write
\[
p_\theta
=
\Pr_{Y\sim\pi_\theta(\cdot\mid x)}
\bigl(R(x,Y)=1\bigr),
\]
and denote the policy conditioned on success by
\[
\pi_\theta^+(y\mid x)
=
\pi_\theta\!\left(y\mid x,\ R(x,Y)=1\right).
\]
\subsection{Recovering supervision from failed unguided groups}
\label{app:rgpo-recovery}


\begin{proposition}[Guided supervision recovery]
\label{prop:app-recovery}
Fix $x$ and a policy snapshot.  Suppose $R\in\{0,1\}$,
$0\leq\delta<1$, and $p<1$.  Conditional on $B$, suppose the $K_g$ guided
candidates are independent and have common success probability $s_B$.  Then
\begin{align}
\Pr(A)
&=(1-p)^K\,
\mathbb{E}\!\left[1-(1-s_B)^{K_g}\mid F\right],
\label{eq:app-recovery-prob}\\
\mathbb{E}[N_{\rm acc}]
&=K_g(1-p)^K\,\mathbb{E}[s_B\mid F].
\label{eq:app-accepted-count}
\end{align}
If $p=1$, both quantities are zero.
\end{proposition}

\begin{proof}
If any unguided response succeeds, the right-hand side of the gate in
Eq.~\eqref{eq:app-gate} is at least $1$ and no binary-reward candidate can be
accepted.  On $F$, the gate threshold is $\delta<1$, so a candidate is accepted
if and only if it succeeds.  Conditional on $B\in F$, therefore,
$N_{\rm acc}\sim\operatorname{Binomial}(K_g,s_B)$.  Since
$\Pr(F)=(1-p)^K$, conditioning on $F$ gives both identities.
\end{proof}

For standard group-centered GRPO with binary reward, an all-failure group has
zero within-group reward variance and hence zero reward-advantage vector.  The
proposition therefore identifies a source of positive supervision precisely on
an event on which this GRPO signal vanishes.  The following proposition also
quantifies how the expected correctness signal shrinks as success becomes rare.

\begin{proposition}[GRPO's correctness signal vanishes linearly in the pass rate]
\label{prop:app-grpo-signal}
Assume $\rho_\theta=\pi_\theta$, $0<p<1$, and define the summed first-step GRPO
score update
\[
\widehat g_{\rm GRPO}
=\sum_{k=1}^K \widehat A_k\nabla_\theta\log\pi_\theta(Y_k\mid x),
\]
where the binary rewards are normalized by their group mean and population
standard deviation and $\widehat A_k=0$ when the group reward variance is zero.
Let $M\sim\operatorname{Binomial}(K,p)$.  Then
\begin{equation}
\mathbb{E}[\widehat g_{\rm GRPO}]
=c_K(p)\,\nabla\log p,
\qquad
c_K(p)
=\frac{\mathbb{E}\sqrt{M(K-M)}}{1-p}
\leq K\sqrt{K-1}\,p.
\label{eq:app-grpo-signal}
\end{equation}
In particular $c_K(p)=O(p)$ as $p\downarrow0$, while the update is exactly zero
on the all-failure event, whose probability is $(1-p)^K$.
\end{proposition}

\begin{proof}
Let $\pi_\theta^+$ and $\pi_\theta^-$ be the policy conditioned on success and
failure, and set
$\mu^\pm=\mathbb{E}_{\pi_\theta^\pm}\nabla\log\pi_\theta(Y\mid x)$.
Conditioned on $M=m$ with $0<m<K$, a successful sample has normalized advantage
$\sqrt{(K-m)/m}$ and a failed sample has normalized advantage
$-\sqrt{m/(K-m)}$.  Therefore
\[
\mathbb{E}[\widehat g_{\rm GRPO}\mid M=m]
=\sqrt{m(K-m)}(\mu^+-\mu^-).
\]
Since
$\nabla p=p\mu^+=- (1-p)\mu^-$, we have
$\mu^+-\mu^-=\nabla p/[p(1-p)]$.  Averaging over
$M\sim\operatorname{Binomial}(K,p)$ gives the equality in
Eq.~\eqref{eq:app-grpo-signal}.  For $1\leq m\leq K-1$,
$\sqrt{m(K-m)}\leq m(K-m)/\sqrt{K-1}$, while both sides vanish for
$m\in\{0,K\}$.  Hence
\[
\mathbb{E}\sqrt{M(K-M)}
\leq \frac{\mathbb{E}[M(K-M)]}{\sqrt{K-1}}
=K\sqrt{K-1}\,p(1-p),
\]
which proves the bound.
\end{proof}


\begin{corollary}[Equal response budget]
\label{cor:equal-budget}
Suppose that, conditional on an all-failure unguided group, the
$K_g$ guided candidates are conditionally independent and each succeeds
with probability $s$. Then the probability of discovering at least one
successful response using $K$ unguided and $K_g$ guided samples is
\begin{equation}
P_{\rm RGPO}
=
1-(1-p)^K(1-s)^{K_g}.
\label{eq:rgpo-budget}
\end{equation}
If the same total budget of $K+K_g$ responses is instead used entirely
for unguided sampling, then
\begin{equation}
P_{\rm GRPO}
=
1-(1-p)^{K+K_g}.
\label{eq:grpo-budget}
\end{equation}
For $K_g\geq1$,
\begin{equation}
P_{\rm RGPO}>P_{\rm GRPO}
\quad\Longleftrightarrow\quad
s>p.
\end{equation}
\end{corollary}

\begin{proof}
Using Eqs.~\eqref{eq:rgpo-budget}--\eqref{eq:grpo-budget},
\[
P_{\rm RGPO}>P_{\rm GRPO}
\iff
(1-s)^{K_g}<(1-p)^{K_g}.
\]
Since $K_g\geq1$, this is equivalent to $s>p$.
\end{proof}

\paragraph{Feedback-dependent proposals.}
When $s_B$ varies with the failed group, the exact equal-budget condition is
\begin{equation}
\mathbb{E}[(1-s_B)^{K_g}\mid F]<(1-p)^{K_g}.
\label{eq:app-feedback-condition}
\end{equation}
Thus, for $K_g>1$, the comparison depends on the conditional failure
probability of the guided group, rather than only on
$\mathbb{E}[s_B\mid F]$.

\paragraph{General rewards.}
For an arbitrary reward and possibly heterogeneous guided proposals, define
\[
a_m(B)
=\Pr\!\left(
R(x,\widetilde Y_m)>\max_j R(x,Y_j)+\delta\mid B
\right).
\]
If the guided candidates are conditionally independent given $B$, then
\begin{equation}
\Pr(A)=\mathbb{E}_B\!\left[1-\prod_{m=1}^{K_g}(1-a_m(B))\right],
\qquad
\mathbb{E}[N_{\rm acc}]=\mathbb{E}_B\sum_{m=1}^{K_g}a_m(B).
\label{eq:app-general-acceptance}
\end{equation}
The second identity follows by linearity even without candidate independence.

\subsection{Transfer to the original prompt}
\label{app:rgpo-transfer}


\begin{theorem}[Binary-reward transfer with target mismatch]
\label{thm:app-transfer-binary}
Suppose $0<p_\theta<1$. Let $\mu$ be a fixed distribution supported on
successful responses, and define
\begin{equation}
L_\mu(\theta)
=
-\mathbb{E}_{Y\sim\mu}\log \pi_\theta(Y\mid x),
\qquad
D_\theta
=
\mathrm{KL}\!\left(
\mu\,\|\,\pi_\theta^+(\cdot\mid x)
\right).
\label{eq:app-transfer-defs}
\end{equation}
For any $\theta_0,\theta_1$, let
\[
\Delta L
=
L_\mu(\theta_0)-L_\mu(\theta_1).
\]
Then
\begin{equation}
\log\frac{p_{\theta_1}}{p_{\theta_0}}
=
\Delta L+D_{\theta_1}-D_{\theta_0}
\geq
\Delta L-D_{\theta_0}.
\label{eq:app-transfer-binary}
\end{equation}
In particular, $p_{\theta_1}>p_{\theta_0}$ whenever
$\Delta L>D_{\theta_0}$. If
$\mu=\pi_{\theta_0}^+(\cdot\mid x)$, any strict decrease in
$L_\mu$ implies $p_{\theta_1}>p_{\theta_0}$.
\end{theorem}

\begin{proof}
On the support of $\mu$,
\[
\pi_\theta(y\mid x)
=
p_\theta\,\pi_\theta^+(y\mid x).
\]
Hence
\begin{align}
L_\mu(\theta)
&=
-\log p_\theta
-\mathbb{E}_{Y\sim\mu}
  \log \pi_\theta^+(Y\mid x) \\
&=
-\log p_\theta
+H(\mu)
+D_\theta,
\end{align}
where
$H(\mu)=-\mathbb{E}_{Y\sim\mu}\log\mu(Y)$.
Thus
\[
\log p_\theta
=
H(\mu)+D_\theta-L_\mu(\theta).
\]
Taking the difference between $\theta_1$ and $\theta_0$ gives the equality in
\eqref{eq:app-transfer-binary}; the inequality follows from
$D_{\theta_1}\geq0$.
\end{proof}



\begin{proposition}[Transfer under nonnegative rewards]
\label{prop:app-transfer-general}
Fix a prompt $x$ and let $R(x,y)\geq 0$. Define
\[
J_x(\theta)
=
\mathbb{E}_{Y\sim\pi_\theta(\cdot\mid x)}[R(x,Y)],
\]
and assume $J_x(\theta_i)>0$ for $i\in\{0,1\}$. Let $\mu$ be a fixed
distribution such that, for every $y$ in its support,
\[
\pi_{\theta_i}(y\mid x)R(x,y)>0,
\qquad i\in\{0,1\}.
\]
Define the reward-weighted response distribution
\begin{equation}
Q_\theta^R(y\mid x)
=
\frac{\pi_\theta(y\mid x)R(x,y)}{J_x(\theta)},
\qquad
D_\theta^R
=
\mathrm{KL}\!\left(\mu\,\|\,Q_\theta^R(\cdot\mid x)\right).
\end{equation}
For
\[
\Delta L
=
L_\mu(\theta_0)-L_\mu(\theta_1),
\]
we have
\begin{equation}
\log\frac{J_x(\theta_1)}{J_x(\theta_0)}
=
\Delta L
+
D_{\theta_1}^R-D_{\theta_0}^R
\geq
\Delta L-D_{\theta_0}^R.
\label{eq:app-transfer-general}
\end{equation}
\end{proposition}

\begin{proof}
Expanding the KL divergence,
\begin{align}
D_\theta^R
&=
\mathbb{E}_{Y\sim\mu}
\left[
\log
\frac{\mu(Y)J_x(\theta)}
     {\pi_\theta(Y\mid x)R(x,Y)}
\right] \\
&=
-H(\mu)
+L_\mu(\theta)
-\mathbb{E}_{Y\sim\mu}\log R(x,Y)
+\log J_x(\theta).
\end{align}
The first and third terms are independent of $\theta$. Taking the difference
between $\theta_1$ and $\theta_0$ therefore gives the equality in
\eqref{eq:app-transfer-general}. The inequality follows from
$D_{\theta_1}^R\geq0$.
\end{proof}

\begin{corollary}[Local alignment when the mismatch is small]
\label{cor:app-local-alignment}
Let $f_\theta(y)=\nabla_\theta\log\pi_\theta(y\mid x)$ and suppose
$\|f_\theta(y)\|\leq G$ on the supports of $\mu$ and $Q_\theta^R$.  Define
$u=\nabla\log J_x(\theta)$ and
$v=\mathbb{E}_{Y\sim\mu}f_\theta(Y)=-\nabla L_\mu(\theta)$.  Then
\begin{equation}
\|v-u\|
\leq G\sqrt{2D_\theta^R}.
\label{eq:app-local-alignment}
\end{equation}
In particular, small reward-weighted proposal mismatch makes the auxiliary SFT
direction close to the gradient of log expected reward.
\end{corollary}

\begin{proof}
Differentiating $J_x$ gives
$u=\mathbb{E}_{Q_\theta^R}f_\theta$.  Hence
\[
\|v-u\|
\leq 2G\,\mathrm{TV}(\mu,Q_\theta^R)
\leq G\sqrt{2\,\mathrm{KL}(\mu\|Q_\theta^R)},
\]
where the last step is Pinsker's inequality.
\end{proof}

\subsection{Adaptive assistance from unguided competence}
\label{app:rgpo-scheduler}

With index repeated visits to one prompt by $t$,  let $\mathcal H_t$ contain the
history before its unguided samples, including the current policy and hint
level.  Let $p_t$ be its correctness probability under the unguided decoding
law, and let $C_t$ indicate at least one correct response among $K$
conditionally independent unguided samples.  The scheduler is
\begin{equation}
g_{t+1}=\operatorname{clip}(g_t+1-2C_t,0,N).
\label{eq:app-scheduler}
\end{equation}

\begin{proposition}[Conditional drift of the hint level]
\label{prop:app-drift}
For $0<g_t<N$,
\begin{equation}
\mathbb{E}[g_{t+1}-g_t\mid\mathcal H_t]
=2(1-p_t)^K-1.
\label{eq:app-drift}
\end{equation}
At $g_t=0$ the drift equals $(1-p_t)^K$, and at $g_t=N$ it equals
$-[1-(1-p_t)^K]$.  Every positive hint level has negative drift whenever
\begin{equation}
p_t>p_\star:=1-2^{-1/K}.
\end{equation}
\end{proposition}

\begin{proof}
The unguided group contains no successful rollout with probability
$(1-p_t)^K$. For $0<g_t<N$, this event increases $g_t$ by one, while its
complement decreases $g_t$ by one, giving
\eqref{eq:app-drift}. The boundary expressions follow from clipping.
The threshold $p_\star$ is obtained by solving
$2(1-p_t)^K-1<0$.
\end{proof}

\begin{corollary}[Expected time to remove assistance]
\label{cor:app-fading}
Starting from $g_0=g$, let
$T_0=\inf\{t\geq0:g_t=0\}$.  If
$p_t\geq\tau>p_\star$ almost surely for all $t<T_0$, then
\begin{equation}
\mathbb{E}[T_0\mid\mathcal H_0]
\leq
\frac{g}{1-2(1-\tau)^K}.
\label{eq:app-fading-time}
\end{equation}
\end{corollary}

\begin{proof}
Set
\[
d=1-2(1-\tau)^K>0.
\]
By Proposition~\ref{prop:app-drift}, for all $t<T_0$,
\[
\mathbb{E}[g_{t+1}-g_t\mid\mathcal H_t]\leq -d.
\]
Applying this bound to the stopped process gives
\[
0\leq
\mathbb{E}[g_{T_0\wedge n}\mid\mathcal H_0]
\leq
g-d\,\mathbb{E}[T_0\wedge n\mid\mathcal H_0].
\]
Letting $n\to\infty$ and using monotone convergence yields
\eqref{eq:app-fading-time}.
\end{proof}


\subsection{Learning time in a two-response model}
\label{app:rgpo-learning-time}

\paragraph{Policy and update.}
This section analyzes a stylized two-response model intended to isolate the sparse-reward mechanism. Given one prompt with exactly two responses, one correct and one incorrect, the policy is parameterized by a scalar logit $z_t$ with
$p_t=\sigma(z_t)$.  The rollout and likelihood laws coincide.  At each visit,
we draw
\[
n_t\sim\operatorname{Binomial}(K,p_t),
\qquad K\geq2,
\]
where $n_t$ is the number of correct unguided responses.  Sampled advantages
and accepted targets are held fixed during one fresh SGD step.  For a
population-standard-deviation group normalization with stabilizer
$\epsilon\geq0$, define
\begin{equation}
V_K(n)=\frac nK\left(1-\frac nK\right),
\qquad
b_{K,\epsilon}(n)=
\begin{cases}
\displaystyle
\frac{V_K(n)}{\sqrt{V_K(n)}+\epsilon}, & 0<n<K,\\[6pt]
0, & n\in\{0,K\}.
\end{cases}
\label{eq:app-grpo-increment}
\end{equation}
With learning rate $\eta>0$, no KL penalty or momentum, and a mean auxiliary
loss over accepted responses from the current visit, the scalar update is
\begin{equation}
z_{t+1}
=z_t+\eta\left[
 b_{K,\epsilon}(n_t)
 +\lambda_t\mathbf 1_{A_t}(1-p_t)
\right],
\qquad
\lambda_t\geq\lambda_{\min}>0.
\label{eq:app-binary-update}
\end{equation}
For binary reward and $0\leq\delta<1$, $A_t$ occurs exactly when
$n_t=0$, $g_t>0$, and at least one guided response succeeds.  We use the
scheduler in Eq.~\eqref{eq:app-scheduler} with $C_t=\mathbf 1\{n_t>0\}$.
Every increment in Eq.~\eqref{eq:app-binary-update} is nonnegative, so $p_t$
is nondecreasing in this model.

\paragraph{Useful-hint condition.}
\label{assump:app-useful-guidance}
Fix $0<p_0<\tau<1$. There exist a hint level
$g_*\in\{1,\ldots,N\}$ and a constant $c>0$ such that, whenever
$p_t<\tau$ and $g_t\geq g_*$, conditional on an all-failure unguided group,
\begin{equation}
\Pr(A_t\mid\mathcal H_t,\text{all-failure})\geq c.
\label{eq:app-useful-guidance}
\end{equation}
If the $K_g$ guided candidates are conditionally independent and each succeeds
with probability at least $s_{\min}>0$, then one may take
$c=1-(1-s_{\min})^{K_g}$.

No condition is imposed for $g_t<g_*$. In particular, $g_*=N$ allows the
condition to hold only when the full rationale is revealed.

Let
\[
T_\tau=\inf\{t\geq0:p_t\geq\tau\},
\]
and define
\begin{equation}
\ell=g_*+1,
\qquad
\beta=(1-\tau)^{K\ell}c,
\qquad
m=
\left\lceil
\frac{
\operatorname{logit}(\tau)-\operatorname{logit}(p_0)
}{
\eta\lambda_{\min}(1-\tau)
}
\right\rceil .
\label{eq:app-learning-constants}
\end{equation}

\begin{lemma}[Block progress]
\label{lem:app-block-progress}
From any history prior to $T_\tau$, a block of $\ell$ consecutive visits
contains either an accepted guided response or a visit at which
$p_t\geq\tau$ with conditional probability at least $\beta$.
\end{lemma}

\begin{proof}
While $t<T_\tau$, an unguided group fails with conditional probability at
least
\[
f:=(1-\tau)^K.
\]
Consider a block in which the target is not reached and no guided response is
accepted during its first $g_*$ visits. If all of these unguided groups fail,
the scheduler increases the hint level at each visit, so that the next visit
has $g_t\geq g_*$. An additional all-failure group then yields an accepted
guided response with conditional probability at least $c$ by
Assumption~\ref{assump:app-useful-guidance}. Hence
\[
\Pr(\text{progress within the block}\mid\mathcal H_t)
\geq
f^{g_*+1}c
=
\beta.
\]
\end{proof}

\begin{theorem}[Rare-success scaling]
\label{thm:app-rare-success}
Under Eq.~\eqref{eq:app-binary-update} and
Assumption~\ref{assump:app-useful-guidance},
\begin{align}
\mathbb{E}[T_\tau]
&\leq
\frac{\ell m}{\beta},
\label{eq:app-learning-mean}\\
\Pr(T_\tau>\ell n)
&\leq
\Pr\!\left\{
\operatorname{Binomial}(n,\beta)<m
\right\}
\label{eq:app-learning-tail}\\
&\leq
\exp\!\left(
-\frac{(n\beta-m)^2}{2n\beta}
\right),
\qquad n\beta>m.
\notag
\end{align}
For group-centered GRPO using the same response budget
$b=K+K_g$ per visit and no auxiliary update,
\begin{equation}
\mathbb{E}[T_\tau^{\mathrm{GRPO}}]
\geq
\frac{1}{
1-(1-p_0)^b-p_0^b
}
\geq
\frac{1}{bp_0}.
\label{eq:app-plain-waiting}
\end{equation}
Therefore, for fixed
$K,K_g,g_*,\tau,\eta,\lambda_{\min}$ and $c$,
\begin{equation}
\mathbb{E}[T_\tau]
=
O\!\left(\log\frac{1}{p_0}\right),
\qquad
\mathbb{E}[T_\tau^{\mathrm{GRPO}}]
=
\Omega\!\left(\frac{1}{p_0}\right)
\quad
\text{as }p_0\downarrow0.
\label{eq:app-asymptotic-rates}
\end{equation}
\end{theorem}

\begin{proof}
For $t<T_\tau$, an accepted guided response increases the logit by at least
\[
d:=\eta\lambda_{\min}(1-\tau),
\]
while all other terms in Eq.~\eqref{eq:app-binary-update} are nonnegative.
Thus $m$ accepted responses suffice to increase the logit from
$\operatorname{logit}(p_0)$ to $\operatorname{logit}(\tau)$.

Partition the visits into blocks of length $\ell$. By
Lemma~\ref{lem:app-block-progress}, before $T_\tau$ each block has conditional
probability at least $\beta$ of producing an accepted response or reaching the
target. The resulting block indicators stochastically dominate i.i.d.\
Bernoulli$(\beta)$ variables. On the event $\{T_\tau>\ell n\}$, no block can
reach the target, and therefore fewer than $m$ blocks can contain an
acceptance. This gives
\eqref{eq:app-learning-tail}. The expectation bound follows by comparison with
a negative-binomial random variable with parameters $(m,\beta)$.

For GRPO, the policy does not change until the first group containing both a
successful and an unsuccessful response. Prior to this event, the success
probability remains $p_0$. Its waiting time is geometric with parameter
\[
q
=
1-(1-p_0)^b-p_0^b.
\]
Since reaching $\tau>p_0$ requires at least one nonzero update,
$\mathbb{E}[T_\tau^{\mathrm{GRPO}}]\geq q^{-1}$. Moreover,
$q\leq bp_0$, which gives \eqref{eq:app-plain-waiting}. Finally,
\[
\operatorname{logit}(\tau)-\operatorname{logit}(p_0)
=
\log\frac{1}{p_0}+O(1),
\]
and hence $m=O(\log(1/p_0))$ as $p_0\downarrow0$.
\end{proof}

\begin{corollary}[Time to competence without guidance]
\label{cor:app-learn-and-fade}
Under the conditions of Theorem~\ref{thm:app-rare-success}, suppose
$\tau>p_\star=1-2^{-1/K}$, and define
\[
T_{\rm zero}
=
\inf\{t\geq0:p_t\geq\tau,\ g_t=0\}.
\]
Then
\begin{equation}
\mathbb{E}[T_{\rm zero}]
\leq
\frac{\ell m}{\beta}
+
\frac{N}{1-2(1-\tau)^K}.
\label{eq:app-learn-and-fade}
\end{equation}
\end{corollary}

\begin{proof}
Equation~\eqref{eq:app-binary-update} implies that $p_t$ is nondecreasing, so
$p_t\geq\tau$ for all $t\geq T_\tau$. Applying
Corollary~\ref{cor:app-fading} at $T_\tau$ and using
$g_{T_\tau}\leq N$ gives the result.
\end{proof}
\section{Experimental Setup}
\label{app:experimental-setup}
\subsection{Language-Model Experiments}
\label{app:lm-setup}
We implement all language-model post-training experiments in \texttt{verl} using the GRPO advantage estimator and vLLM for rollout generation.  Training uses a global batch size of 256 prompts. For each prompt, we sample four on-policy responses. Responses are allowed up to 4,096 tokens, while reference hints are truncated to at most 512 tokens. Because guided generation appends additional context to the original problem, we allow a maximum re-prompted sequence length of 6,000 tokens. Prompts longer than the configured 1,024-token input limit are filtered before training rather than silently truncated.

\begin{table}[!ht]
\centering
\caption{\textbf{Training configuration for the language-model and vision-language-model experiments.}
We report the optimization and sequence-length settings shared across RGPO and its corresponding baselines.}
\label{tab:training_hparams}
\small
\begin{tabular}{lcc}
\toprule
\textbf{Hyperparameter} & \textbf{LM} & \textbf{VLM} \\
\midrule
Training epochs          & 3                  & 3 \\
Batch size               & 256                & 512 \\
Rollouts per prompt      & 4                  & 4 \\
Learning rate            & $1\times10^{-6}$   & $1\times10^{-6}$ \\
Max. prompt length       & 1,024              & 768 \\
Max. response length     & 4,096              & 1,024 \\
Max. hint length         & 512                & 512 \\
Max. guided prompt length& 6,000              & 2,560 \\
Entropy coefficient      & $1\times10^{-3}$   & 0 \\
KL regularization        & None               & None \\
Precision                & BF16               & BF16 \\
Rollout samples $K$      & 4 + 4                  & 4 + 4 \\
\bottomrule
\end{tabular}
\end{table}

We optimize the actor with a learning rate of $1\times10^{-6}$ for three epochs. The PPO minibatch size is 256, corresponding to one optimizer minibatch per global training batch, with gradient accumulation through smaller per-GPU microbatches. Training uses bfloat16 parameters, FSDP2 sharding, gradient checkpointing, and fused model kernels. We do not use a KL penalty either in the actor objective or as part of the reward. We use a small entropy coefficient of $10^{-3}$ for the language-model experiments. Rollouts are generated with tensor-parallel size 1 and four samples per prompt. We use an actor microbatch size of 8 and a rollout log-probability microbatch size of 16 per GPU.

\subsection{Vision-Language Training}

The vision-language experiments use the same GRPO/RGPO implementation, with multimodal inputs enabled throughout rollout and optimization.  As in the language-model experiments, we disable a learned reward model and use a programmatic verifier to determine response correctness. We use a larger global batch of 512 prompts and sample four responses per prompt. The maximum text prompt length is 768 tokens, the generated response length is 1,024 tokens, and at most 512 tokens of the reference rationale are exposed as guidance. Guided re-prompting is allowed up to 2,560 tokens to accommodate the original multimodal prompt, partial rationale, and refinement instruction. Over-length input prompts are filtered rather than truncated.

The actor is optimized for three epochs with learning rate $1\times10^{-6}$ and PPO minibatch size 512. Training uses FSDP2, bfloat16 precision, gradient checkpointing, and fused kernels. We use neither KL regularization nor an entropy bonus in the vision-language experiments. Rollouts are generated with vLLM using tensor-parallel size 1 and four samples per prompt. Chunked prefill is enabled for the VLM rollout engine. The actor microbatch size is 32 per GPU and the rollout log-probability microbatch size is 64 per GPU. 

\subsection{Baselines}
\label{app:baselines}

\paragraph{On-policy RLVR.}
Our primary on-policy baseline is \textbf{GRPO}~\citep{shao2024deepseekmath}, which estimates advantages by comparing rewards among multiple responses sampled for the same prompt and avoids training a separate value model. We also evaluate \textbf{REINFORCE++}~\citep{hu2025reinforcepp} in the learning-efficiency experiments. REINFORCE++ is a critic-free REINFORCE variant that uses global advantage normalization together with PPO-style stabilization techniques. Comparing RGPO against both objectives tests whether its gains arise from rationale-guided exploration rather than from a particular underlying RLVR optimizer.

\paragraph{Off-policy supervision and guided exploration.}
We compare against several recent methods that address sparse or uninformative on-policy rewards by introducing external reasoning signals. \textbf{LUFFY}~\citep{yan2025luffy} mixes off-policy expert reasoning traces with on-policy rollouts and uses policy shaping to control the influence of demonstrations. \textbf{CHORD}~\citep{zhang2025chord} incorporates expert solutions as a dynamically weighted supervised objective within on-policy RL, gradually balancing imitation and exploration. These methods directly optimize toward externally provided expert trajectories.

A second family instead uses external information to make difficult rollouts more tractable. \textbf{GHPO}~\citep{liu2025ghpo} adaptively adjusts the amount of guidance according to the model's current ability and combines imitation-oriented learning on difficult examples with RL on more accessible ones. \textbf{QuestA}~\citep{li2025questa} augments difficult questions with partial solutions during RL, effectively reducing their difficulty and increasing the probability of obtaining informative rollouts. \textbf{BREAD}~\citep{zhang2025bread} introduces short expert prefixes when self-generated trajectories fail and branches new rollouts from these expert-provided intermediate states. \textbf{HINT}~\citep{wang2025hint} supplies heuristic guidance to ineffective rollouts while explicitly seeking to keep the guided trajectories close to the current policy. For the additional prompt-format analysis, we also compare with \textbf{AHPO} from MM-HELIX~\citep{zhao2025mm}, which combines adaptive off-policy supervision with online RL.


\paragraph{Multimodal reasoning baselines.}
For controlled multimodal comparisons, we include the supervised reasoning variants from MINT-CoT~\citep{chen2025mintcot}: \textbf{Original Image CoT SFT}, \textbf{Bounding-Box CoT SFT}, and \textbf{Text-only CoT SFT}. These baselines use the same Qwen2-VL backbone family and directly supervise the model with different forms of multimodal reasoning traces, providing a comparison between rationale supervision and rationale-guided RL. We further report representative open multimodal reasoning systems. \textbf{R1-VL}~\citep{zhang2025r1vl} improves multimodal reasoning through StepGRPO with step-wise rule-based rewards; \textbf{Open-R1-Multimodal}~\citep{open-r1-multimodal} provides an open implementation of GRPO-style multimodal RL; \textbf{Mulberry}~\citep{yao2024mulberry} constructs explicit multimodal reasoning trajectories using collective Monte Carlo tree search and trains on them with supervised learning; and \textbf{MM-Eureka}~\citep{meng2025mmeureka} applies large-scale rule-based reinforcement learning to multimodal reasoning.

Finally, we include strong general-purpose VLMs as reference points: \textbf{LLaVA-OneVision-Qwen2-7B-OV}~\citep{li2024llava}, \textbf{InternVL2-8B}~\citep{chen2024internvl}, \textbf{InternVL2-8B-MPO}~\citep{wang2024enhancing}, \textbf{DeepSeek-VL2}~\citep{wu2024deepseekvl2}, and \textbf{Qwen2.5-VL-7B-Instruct}~\citep{bai2025qwen25vl}. We report these systems to place RGPO's multimodal reasoning performance in the broader landscape rather than as strictly matched baselines, since they use different backbones, training corpora, and post-training recipes.

\paragraph{MM-HELIX experiment.}
To examine whether RGPO requires a high-quality expert chain-of-thought, we additionally train Qwen2.5-VL-7B-Instruct on the 24-Points subset of MM-HELIX-100K~\citep{zhao2025mm}. MM-HELIX constructs these problems using a deterministic solver and converts intermediate solver states into programmatically generated reasoning traces. In our RGPO experiment, we use only these solver-derived traces as rationale scaffolds; we do not use the longer natural-language rationales rewritten by Qwen3-235B. We further consider a feedback-only variant, \textbf{RGPO w/o CoT}, in which no solution rationale is provided. Instead, unsuccessful responses are refined using diagnostic feedback from the deterministic 24-Points evaluator, which checks output format, use of the required numbers, arithmetic correctness, and evaluability. This setting tests whether RGPO can benefit from weaker forms of structured guidance that do not require an expert-generated reasoning trajectory.

\subsection{Controlled Simulation Setup}


\paragraph{Synthetic reasoning problems.}
Each problem consists of a short sequence of reasoning decisions followed by a final answer. Decision types are shared across problems, so improving the policy on a particular type of decision can transfer to other problems in which the same decision recurs. The simulator contains 40 such decision types and 16 possible actions for each decision. We divide the decision types into 32 \emph{easy} and 8 \emph{hard} types. For an easy decision, one or two actions are valid and the policy is initialized to assign substantially higher probability to these actions. Its initial probability of making a valid easy decision is therefore approximately $0.78$--$0.89$. For a hard decision, only one of the 16 actions is valid and the policy is initialized uniformly, giving an initial success probability of $1/16$. Easy problems contain two to five easy decisions. Hard problems contain five to seven decisions, four of which are hard. Because a reasoned response is correct only when \emph{every} decision is valid, obtaining a complete solution to a hard problem by unguided sampling is unlikely at initialization. Hard decision types do not occur in easy problems, preventing the easy subset from implicitly teaching the difficult decisions. We independently sample 512 training problems and 256 held-out test problems for each random seed, with hard problems comprising $25\%$ of both splits.

\paragraph{Responses and reward.}
The policy can either produce a reasoning trajectory or answer the problem directly. In reasoning mode, it samples one action for each decision in the problem, and the response receives reward 1 only if every sampled action is valid. In direct-answer mode, it chooses among 100 possible answers and receives reward 1 only for the gold answer. The initial probability of choosing the direct-answer mode is $0.2$; the answer distribution itself is initially uniform. The direct-answer path serves two purposes. First, it provides a weak alternative to multi-step reasoning, analogous to guessing an answer without constructing a valid solution. Second, when the complete reference is revealed, the gold answer becomes visible, allowing us to study whether a method learns to rely on the additional information rather than improving its unguided reasoning policy.

\paragraph{Reference rationales and hints.}
Each problem is paired with a reference rationale containing one recommended action for every reasoning decision. To represent an informative but imperfect rationale, we systematically corrupt the reference for three of the eight hard decision types: whenever one of these types occurs, the reference recommends the same fixed invalid action. The remaining reference steps are correct. The reference can be revealed at three levels, $g\in\{0,1,2\}$. At $g=0$, no rationale is provided. At $g=1$, the first half of the reference reasoning trajectory is available, and at $g=2$, the complete trajectory and final answer are visible. At a revealed reasoning step, the generated trajectory follows the reference action with probability $\rho=0.7$; otherwise, the policy generates the step from its own action distribution. This construction allows guidance to influence generation without completely determining the resulting trajectory.
\begin{table}[!ht]
\centering
\caption{\textbf{Controlled simulation configuration.}
We construct the environment so that hard problems rarely yield a correct
unguided trajectory at initialization, while the reference rationale provides
useful but imperfect guidance.}
\label{tab:sim-params}
\small
\begin{tabularx}{\columnwidth}{@{}lX@{}}
\toprule
\textbf{Parameter} & \textbf{Value} \\
\midrule
Decision types
    & 40 total: 32 easy and 8 hard \\

Actions per decision
    & 16 \\

Problem composition
    & Easy: 2--5 easy decisions; hard: 5--7 decisions with 4 hard decisions \\

Train / test problems
    & 512 / 256, with 25\% hard problems \\

Reference corruption
    & 3 of the 8 hard decision types have systematically incorrect reference actions \\

Hint levels
    & $g\in\{0,1,2\}$: no hint, partial rationale, or full rationale \\

Hint-following probability
    & $\rho=0.7$ \\

Unguided / guided rollouts
    & $K=4$ / $K'=4$ \\

Initial direct-answer probability
    & $0.2$ \\

Acceptance margin
    & $\delta=0$ \\

Auxiliary-loss weight
    & $\lambda=1.0$ \\

Batch size
    & 32 problems \\

Training
    & 3 epochs (16 updates per epoch) \\

Learning rate
    & 64 \\

Evaluation
    & 5 random seeds; all methods evaluated without hints \\
\bottomrule
\end{tabularx}
\end{table}
\paragraph{Training strategies.}
All methods use the same policy, training problems, reward function, and optimization budget as described in Table \ref{tab:sim-params}. Training proceeds for three epochs with minibatches of 32 problems, corresponding to 16 parameter updates per epoch. We use plain gradient ascent with learning rate 64 and no KL regularization. Each unguided group contains $K=4$ responses. \textbf{GRPO} uses only the four unguided responses. Advantages are computed by normalizing rewards within each group; when every response receives the same reward, the group contributes no policy-gradient update.

\textbf{RGPO} begins with the same unguided GRPO update. All problems start with $g=0$. After each epoch, the hint level of a problem is increased by one if none of its unguided responses was correct during that epoch and decreased by one otherwise, subject to $0\leq g\leq2$. For a problem with $g>0$, we sample an additional $K'=4$ guided responses. A guided response is retained only if it is correct while all corresponding unguided responses are incorrect. Accepted responses are then optimized under the \emph{original problem without the hint}. The auxiliary supervised loss is averaged over accepted guided responses in the current minibatch and weighted by $\lambda=1$.

We include several variants to isolate how the same reference information is used. \textbf{RGPO w/o gate} distills every guided response rather than retaining only verified improvements. \textbf{RGPO, full hint always} follows the RGPO procedure but provides the complete reference for every problem throughout training. \textbf{RL + SFT(ref)} combines unguided GRPO with supervised learning directly on the reference rationale for every problem, while \textbf{RL + SFT(ref) on failure} applies the reference-SFT objective only when all unguided responses for a problem fail. \textbf{Hint-conditioned RL} applies an additional GRPO update to responses generated in the adaptively hinted context rather than distilling them under the original problem. Finally, \textbf{Full hint} performs GRPO entirely on responses generated with the complete reference available.

All methods are evaluated on the original problems \emph{without hints}. Thus, performance measures what has transferred to the unguided policy rather than what the model can accomplish while the reference remains available.

\begin{equation*}
p_\theta(x)=P(\text{reason})\prod_{j=1}^{L_x}\pi_\theta\big(A(k(x,j))\big)+P(\text{direct})\,\mathrm{softmax}(\psi_x)_{a_x},
\end{equation*}
where $A(k)$ denotes the set of valid actions for decision type $k$ and $a_x$ is the gold answer. We report the mean of $p_\theta(x)$ over a split as accuracy and the mean of $-\log p_\theta(x)$ as negative log-likelihood (NLL). This removes sampling noise from evaluation and makes it possible to examine changes in the policy even when the probability of completing an entire hard problem remains small.
\subsection{Hardware configuration}

 All experiments were conducted on high-performance machines equipped with Intel Xeon CPUs and NVIDIA GPUs. Specifically, we utilized two machine configurations: (1) INTEL(R) XEON(R) PLATINUM 8592+ with 2.0Ti RAM and H200 SXM 141GB GPUs and (2) Intel(R) Xeon(R) Platinum 8358 CPU @ 2.60GHz with  503Gi  RAM and  NVIDIA A100 SXM4 80GB GPUs. Although different GPU types were used to balance workload priorities, we ensured that all comparisons across methods were performed on the same hardware configuration for a given model and dataset to eliminate hardware-induced variability and maintain consistency and fairness in evaluation.

\subsection{Prompt templates}

RGPO uses rationales as temporary scaffolds to elicit an improved response. To implement this idea, we design a refinement prompt that conditions the model on three pieces of information: the original problem, a partial solution prefix, and feedback on its previous attempt. The partial solution provides a rationale guide, while the feedback signal indicates whether the previous answer failed due to formatting, accuracy, or both. Importantly, the model is instructed to generate a new standalone solution to the original question, rather than to comment on or revise the previous response. This prevents the model from learning a conversational repair pattern and instead encourages it to internalize the reasoning needed to solve the original task.

\begin{tcolorbox}[
    colback=white!120!blue,
    colframe=white!60!blue,
    title=Refinement prompt,
    fonttitle=\bfseries\large
]
\label{refine-prompt}
\texttt{\{prompt\}}\\
\texttt{This is the (partial) solution without formatting, you can leverage it to check your answer:}\\
\texttt{\{solution\}}\\
\texttt{Previous response:}\\
\texttt{\{response\}}\\
\texttt{Feedback for your previous response:}\\
\texttt{Format score: \{format\_score\}/1.0}\\
\texttt{Accuracy score: \{accuracy\_score\}/1.0}\\
\texttt{Overall score: \{score\}/1.0}\\
\texttt{You should maximize the overall score by improving both the format and accuracy of your answer.}\\
\texttt{Write a new standalone solution to the original question without mentioning the previous responses or feedback. You first think through the reasoning process as an internal monologue, enclosed within <think> </think> tags. Then, provide your final answer enclosed within \textbackslash boxed\{\}. Please strictly follow the format.}
\end{tcolorbox}

The refinement prompt is shown above. Here, \texttt{\{prompt\}} denotes the original question, \texttt{\{solution\}} denotes the currently revealed rationale segment, \texttt{\{response\}} denotes the model's previous unguided response, and the feedback fields are obtained from the verifier reward. The instruction explicitly asks the model to maximize both format and accuracy, while enforcing the same answer format used during training and evaluation. This is because we found in practice that the model can  sometimes respond with the correct option but in the wrong format due to the long refinement-instruction prompt.

For the language-only experiments, we use the following prompt, closely following \cite{wang2025hint}. 

\begin{tcolorbox}[
    colback=white!120!blue,
    colframe=white!60!blue,
    title=DAPO's system prompt,
    fonttitle=\bfseries\large
]
\label{dapo-prompt}
You are a helpful AI assistant. A conversation takes place between the User and the Assistant. The User asks a question, and the Assistant solves it. Please help me solve this question. Wrap only the final answer in \textbackslash boxed\{\}.
\end{tcolorbox}

For training vision-language models, we adopt a concise reasoning format inspired by DeepSeek-R1. The model is instructed to place its reasoning inside \texttt{<think>} and \texttt{</think>} tags and to wrap the final answer in \texttt{\textbackslash boxed\{\}}.  Compared with the Mulberry prompt used in MINT-CoT, this format is simpler and imposes fewer structural constraints on the model output. In the appendix, we further analyze how this prompt design affects learning dynamics and final performance.
\begin{tcolorbox}[
    colback=white!120!blue,
    colframe=white!60!blue,
    title=DeepSeek-R1's system prompt,
    fonttitle=\bfseries\large
]
\label{deepseek-prompt}
You first think through the reasoning process as an internal monologue, enclosed within <think> </think> tags. Then, provide your final answer enclosed within \textbackslash boxed\{\}.
\end{tcolorbox}

This prompt is more concise and straightforward than the Mulberry-style prompt used in the MINT-CoT dataset, which decomposes the response into multiple explicit sections, including image description, rationale, step-by-step reasoning, and final answer. While such a structured format may encourage detailed multimodal reasoning, it also imposes a more complex output schema. 

\begin{tcolorbox}[
    colback=white!120!blue,
    colframe=white!60!blue,
    title=Mulberry's system prompt,
    fonttitle=\bfseries\large
]
\label{mulberry-prompt}
Generate an image description based on the question.

Then, provide a rationale to analyze the question.

Next, generate a step-by-step reasoning process to solve the problem. Ensure the steps are logical and concise.

Finally, provide a concise summary of the final answer in the following format: The final answer is: xxx.

Format your response with the following sections, separated by \#\#\#:

\#\#\# Image Description:

\#\#\# Rationales:

\#\#\# Let's think step by step.

\#\#\# Step 1:

\#\#\# Step 2:

...

\#\#\# The final answer is: 

\end{tcolorbox}

\clearpage
\section{Additional experimental results}

\subsection{Controlled Simulation}

Figure~\ref{fig:nll} shows both no-hint accuracy and NLL on the complete training and test sets and separately on easy and hard problems.  GRPO rapidly learns the easy subset, reaching $99.1\%$ test accuracy, but remains at $0.0\%$ on hard problems after three epochs. The NLL curves in Figure~\ref{fig:nll} make the underlying behavior more visible. Once direct answering is suppressed, solving a hard problem requires producing all four difficult decisions correctly in the same trajectory. Almost all hard-problem groups therefore contain only failures, yielding no relative reward signal for GRPO. The easy decisions continue to improve, but the hard decisions remain close to their initialization. Hint-conditioned RL receives more information than GRPO, but transfers it inefficiently in this short training horizon. Its hard-problem NLL decreases even though final hard accuracy remains only $0.4\%$, indicating that the policy is beginning to improve individual decisions without yet assigning enough probability to a completely correct multi-step trajectory. 

\begin{figure}[!ht]
\centering
\includegraphics[width=.9\columnwidth]{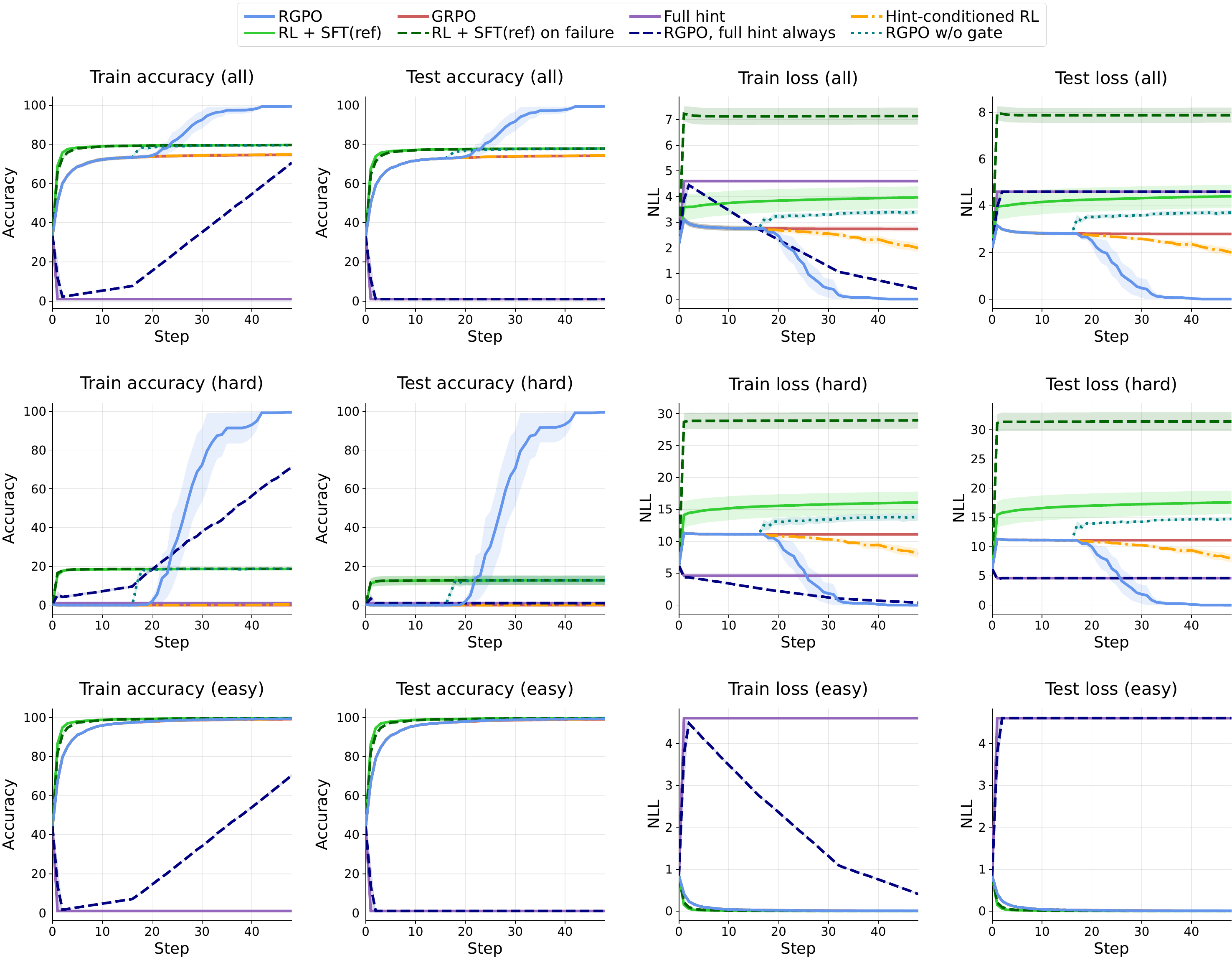}
\caption{\textbf{Full learning dynamics in the controlled simulation.}
Accuracy without hints and corresponding negative log-likelihood on training and held-out problems, shown for all, hard, and easy examples. Curves report mean $\pm$ s.e.\ over five seeds. All methods are evaluated on the original problems without access to the reference rationale.}
\label{fig:nll}
\end{figure}

Both reference-SFT variants reach $12.8\%$ hard test accuracy. Their improvement over GRPO comes from directly supervising the difficult decisions for which the reference is correct. However, the same objective repeatedly reinforces an invalid action for the three hard decision types whose references are systematically corrupted. Problems involving these decisions therefore remain unsolved.  The same behavior appears in \textbf{RGPO w/o gate}. Although its guided trajectories are generated by the policy rather than copied directly from the reference, removing the verifier-based acceptance criterion allows incorrect guided responses to enter the supervised objective. Its hard test accuracy consequently matches the reference-SFT baselines at $12.8\%$. This comparison isolates the role of verification: guidance is useful only when the resulting trajectory is checked before it becomes a training target.

\begin{table}[!ht]
\centering
\caption{\textbf{Final no-hint accuracy (\%) in the controlled simulation after three training epochs.}
Values are mean $\pm$ standard deviation over five seeds. The hard test split is the primary diagnostic: it contains problems for which successful unguided reasoning trajectories are initially extremely unlikely.}
\label{tab:final_sim}
\small
\begin{tabular}{lcccc}
\toprule
\textbf{Method}
& \textbf{Test}
& \textbf{Test (hard)}
& \textbf{Test (easy)}
& \textbf{Train (hard)} \\
\midrule
GRPO
& $74.2\pm1.5$
& $0.0\pm0.0$
& $99.1\pm0.1$
& $0.0\pm0.0$ \\

Full hint
& $1.0\pm0.0$
& $1.0\pm0.0$
& $1.0\pm0.0$
& $1.0\pm0.0$ \\

Hint-conditioned RL
& $74.4\pm1.4$
& $0.4\pm0.8$
& $99.3\pm0.1$
& $0.4\pm0.9$ \\

RL + SFT(ref)
& $77.9\pm0.7$
& $12.8\pm5.6$
& $99.7\pm0.0$
& $18.7\pm1.6$ \\

RL + SFT(ref) on failure
& $77.8\pm0.8$
& $12.8\pm5.6$
& $99.6\pm0.1$
& $18.7\pm1.6$ \\

RGPO w/o gate
& $77.8\pm0.8$
& $12.8\pm5.6$
& $99.5\pm0.1$
& $18.7\pm1.6$ \\

RGPO, full hint always
& $1.0\pm0.0$
& $1.0\pm0.0$
& $1.0\pm0.0$
& $71.2\pm0.5$ \\

\midrule
RGPO
& $\mathbf{99.4\pm0.1}$
& $\mathbf{99.5\pm0.2}$
& $99.3\pm0.1$
& $\mathbf{99.5\pm0.1}$ \\
\bottomrule
\end{tabular}
\end{table} 

Table~\ref{tab:final_sim} reports the corresponding final accuracies after three epochs. The \textbf{Full hint} baseline illustrates a different failure mode. Because the complete reference exposes the gold answer, the policy can obtain perfect reward in the guided context by relying on that information. When the reference is removed for evaluation, however, accuracy falls to $1\%$, exactly the chance rate of selecting among 100 answers. A related effect appears for \textbf{RGPO, full hint always}. Distilling full-hint responses allows the policy to fit many of the training answers, reaching $71.2\%$ accuracy on hard training problems, but this behavior does not transfer to held-out problems: hard test accuracy remains at $1\%$. Thus, high reward under guidance does not by itself imply that the policy has learned a reasoning strategy that can be executed from the original input.

 RGPO behaves differently because the rationale determines where the policy searches, but not which trajectory is ultimately trusted. A guided response contributes to the auxiliary objective only when it passes the verifier and improves over all unguided responses for the same problem. The accepted response is then optimized under the original problem without the rationale. This produces a rapid change in the hard-problem learning dynamics. Once a verified guided trajectory improves one of the difficult decision types, that improvement transfers to other problems containing the same type. These problems in turn become more likely to produce successful unguided rollouts, allowing the ordinary RL objective to begin contributing useful signal. After three epochs, RGPO reaches $99.5\%$ hard test accuracy and $99.4\%$ overall test accuracy while retaining $99.3\%$ performance on easy problems.
\subsection{Fine-Grained Benchmark Results}

\label{app:fine-grained}
Figure \ref{fig:mathvista} presents a category-level comparison of the base model, GRPO, and RGPO on MathVista. 
\begin{figure}[!ht] 
    \centering
     \includegraphics[width=1\columnwidth]{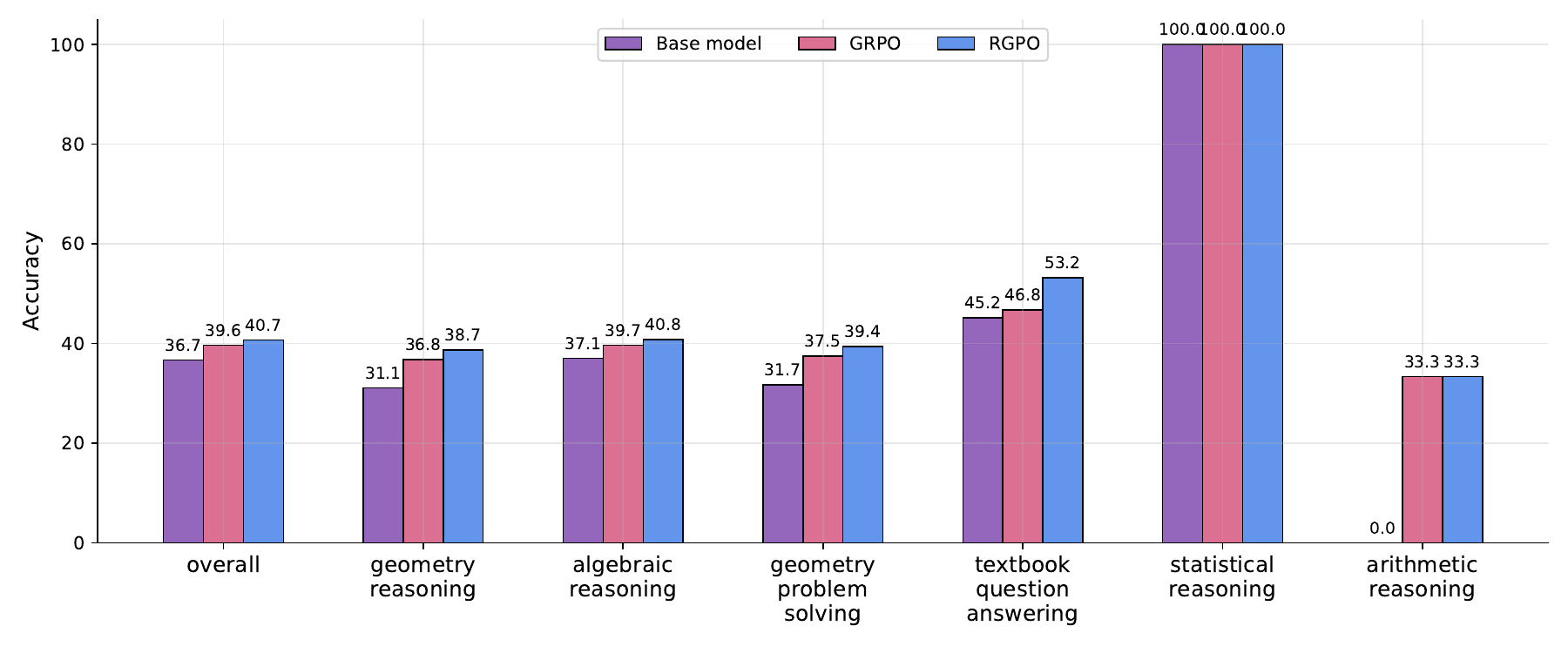}
    \caption{\textbf{Category-wise performance comparison of Qwen2-VL-2B on MathVista.} We compare the base model, GRPO, and RGPO across different mathematical reasoning categories. RGPO achieves the best overall accuracy and consistently improves over both the base model and GRPO on most categories, with the largest gain observed on textbook question answering.}
    \label{fig:mathvista}
\end{figure}
RGPO achieves the highest overall accuracy, improving from 36.7 for the base model and 39.6 for GRPO to 40.7. The gains are consistent across several reasoning categories, including geometry reasoning, algebraic reasoning, geometry problem solving, and textbook question answering. The improvement is especially pronounced on textbook question answering, where RGPO reaches 53.2, outperforming GRPO by 6.4 points and the base model by 8.0 points. 

 This suggests that rationale-guided refinement is particularly helpful for problems requiring structured interpretation and multi-step reasoning. On statistical reasoning, all methods achieve perfect accuracy, while on arithmetic reasoning both GRPO and RGPO improve over the base model. Overall, these results show that RGPO provides a stronger optimization signal than standard RLVR-style training, leading to better no-hint reasoning performance across most MathVista categories.

\begin{figure}[!ht] 
    \centering
     \includegraphics[width=1\columnwidth]{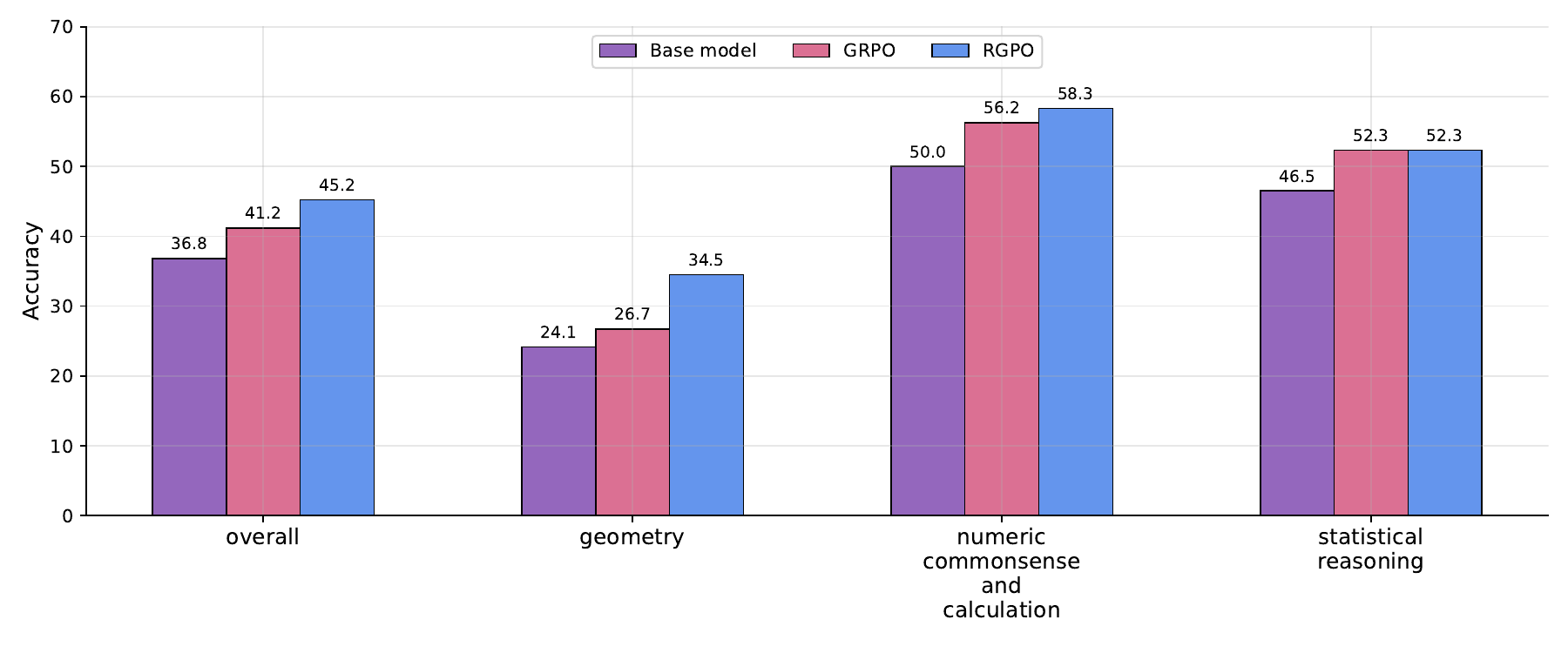}
    \caption{\textbf{Category-wise performance comparison of Qwen2-VL-2B on MMStar:} We compare the base model, GRPO, and RGPO across the mathematical reasoning categories in MMStar. RGPO achieves the highest overall accuracy and improves over both baselines on geometry and numeric commonsense/calculation tasks, while matching GRPO on statistical reasoning.}
    \label{fig:mmstar}
\end{figure}

Figure \ref{fig:mmstar} reports the category-wise performance of the base model, GRPO, and RGPO on MMStar. RGPO achieves the best overall accuracy, improving from 36.8 for the base model and 41.2 for GRPO to 45.2. The largest gain appears in geometry, where RGPO reaches 34.5, outperforming GRPO by 7.8 points and the base model by 10.4 points. RGPO also improves numeric commonsense and calculation accuracy to 58.3, compared with 56.2 for GRPO and 50.0 for the base model. On statistical reasoning, RGPO matches GRPO at 52.3, with both methods outperforming the base model. These results show that RGPO provides consistent benefits on MMStar, especially for visually grounded and calculation-heavy reasoning tasks where rationale-guided refinement can help convert failed attempts into more informative training signals.

\subsection{ Dynamics of Adaptive Rationale Guidance}
Figure \ref{fig:hint} illustrates how RGPO adaptively controls the use of rationale guidance during training. In the first epoch, no hints are provided, allowing the algorithm to estimate the model’s unguided problem-solving ability. Once rationale-guided refinement is activated in the second epoch, the non-empty hint fraction rises to around 20\%, indicating that only a minority of examples require scaffolded assistance. By the third epoch, this fraction drops to roughly 15\%, suggesting that many examples that previously required guidance can now be solved without hints. This behavior supports the central design of RGPO: rationales are not used as always-on supervision, but as temporary scaffolds for difficult examples. As training progresses, the model becomes less dependent on external rationale guidance, making full-solution exposure for every sample unnecessary.

\clearpage

\begin{figure}[t] 
    \centering
     \includegraphics[width=\columnwidth]{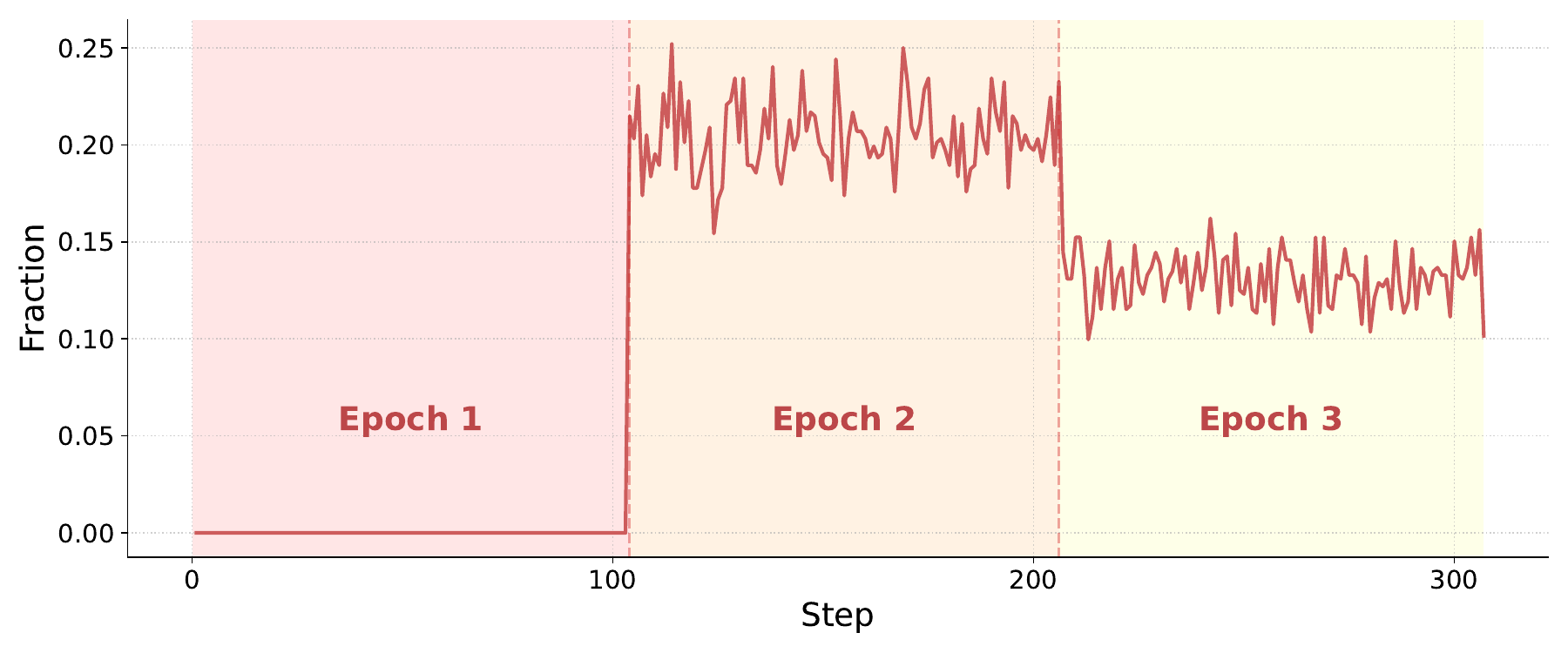}
    \caption{\textbf{Fraction of training examples receiving rationale guidance:} The hint rate remains zero during the probing epoch, rises when RGPO begins guided refinement, and decreases as the model learns to solve more examples independently.}
    \label{fig:hint}
\end{figure}
\subsection{Beyond Expert Rationale Scaffolds}
\label{app:mmhelix}

The experiments above use reference solutions as a convenient source of rationale scaffolds. We next ask whether RGPO fundamentally depends on such expert-written reasoning traces. To separate the benefit of scaffolding from the quality of the rationale source, we evaluate RGPO on the 24-Points task from MM-HELIX~\citep{zhao2025mm}, where guidance can be obtained directly from the underlying programmatic solver. Samples from this dataset are visually presented in Figure \ref{fig:mmhelix-examples}.

\begin{figure*}[!ht]
    \centering
    \includegraphics[width=.85\textwidth]{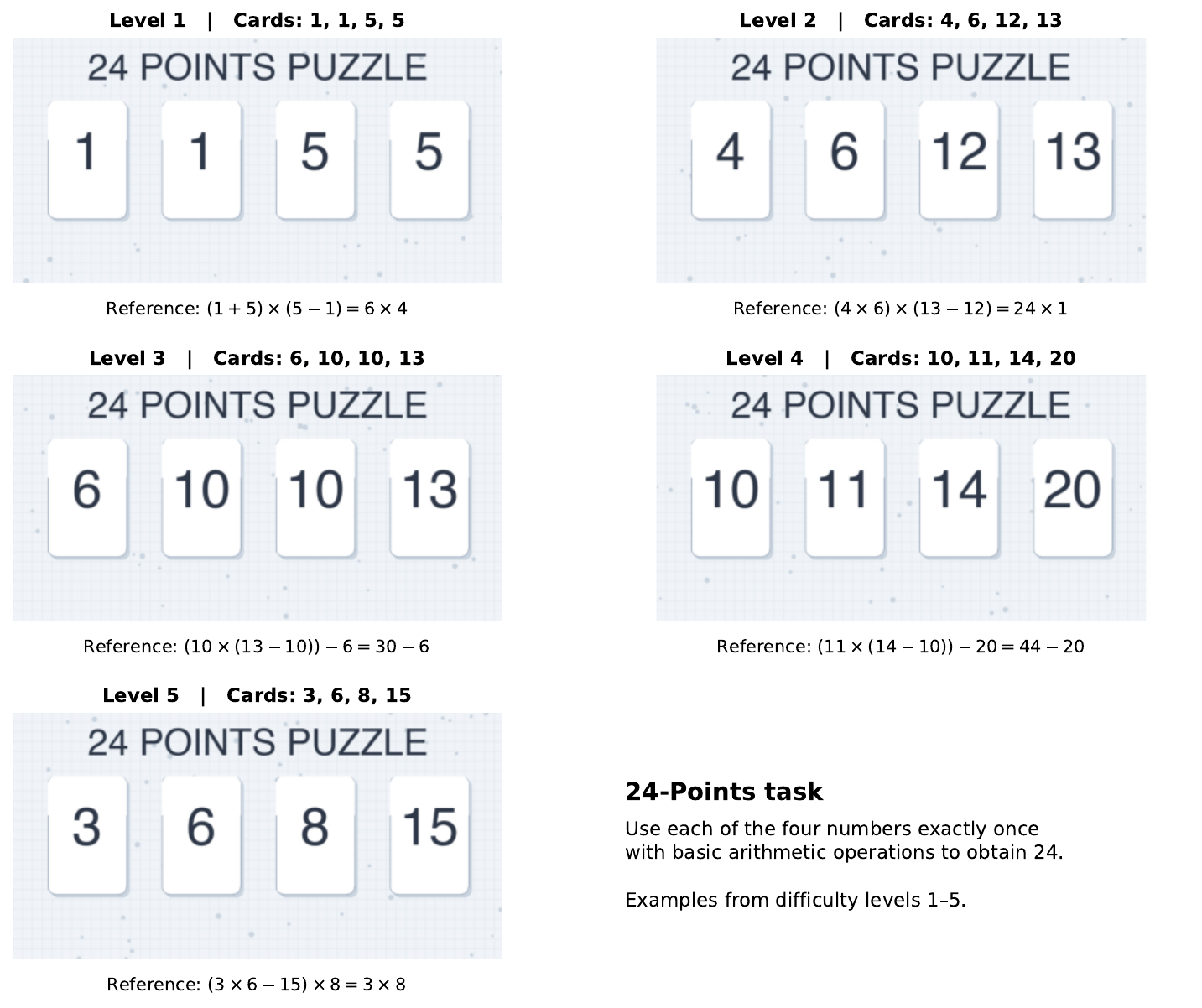}
    \caption{\textbf{Representative 24-Points examples used in our MM-HELIX experiment.}
    Each problem presents four numbers and requires constructing an arithmetic expression that uses every number exactly once and evaluates to 24. 
    }
    \label{fig:mmhelix-examples}
\end{figure*}

For \textbf{RGPO-7B}, we use the mechanical solver traces provided by the MM-HELIX data-generation pipeline as rationale scaffolds. These traces encode valid intermediate reasoning states but are substantially shorter and less natural than the Qwen3-235B-rewritten trajectories used by MM-HELIX-7B. We therefore deliberately give RGPO the less polished source of supervision rather than the expert-rewritten chain-of-thought. We also evaluate \textbf{RGPO-7B w/o CoT}, which removes the solution trace entirely. When an unguided response fails, this variant receives only diagnostic feedback from the deterministic task verifier and uses that feedback to refine its response.
\begin{table*}[!ht]
\centering
\caption{\textbf{Programmatic feedback used by RGPO-7B (w/o CoT).}
The judge supplies diagnostic information about the model's current attempt without revealing a correct solution.
The example uses the input numbers $[2,5,10,12]$.}
\label{tab}
\small
\begin{tabular}{@{}c p{0.18\textwidth} p{0.25\textwidth} p{0.47\textwidth}@{}}
\toprule
\textbf{\#} &
\textbf{Error category} &
\textbf{Trigger condition} &
\textbf{Feedback returned to the model (example, numbers [2, 5, 10, 12])} \\
\midrule

1 &
\textbf{No expression / format} &
No arithmetic operator found in the extracted answer &
\emph{``No arithmetic expression was found in your answer. Combine [2, 5, 10, 12], using each number exactly once, with +, -, $\times$, $\div$ and parentheses to make 24 \ldots''} \\

2 &
\textbf{Number-usage (missing)} &
Used numbers $\neq$ given numbers (multiset) &
\emph{``Number-usage error: you did not use 2. You must use each of [2, 5, 10, 12] exactly once and use no other numbers.''} \\

3 &
\textbf{Number-usage (wrong count / extra)} &
Same, with per-number diagnosis &
\emph{``you used 2 2x but only 1 is available; you used 9, which is not one of the given numbers.''} \\

4 &
\textbf{Wrong value --- too small} &
Correct numbers, result $< 24$ &
\emph{``You used the right numbers, but your expression $(12 - 10) \times 5 + 2$ equals 12, which is 12 short of 24 --- make the result larger.''} \\

5 &
\textbf{Wrong value --- too large} &
Correct numbers, result $> 24$ &
\emph{``\ldots equals 48, which is 24 above 24 --- make the result smaller.''} \\

6 &
\textbf{Division by zero} &
Expression divides by zero &
\emph{``Your expression divides by zero. Rearrange it so that no division by zero occurs. \ldots''} \\

7 &
\textbf{Unevaluable} &
Malformed / invalid characters / unbalanced parens &
\emph{``Your expression could not be evaluated. Use only [2, 5, 10, 12] with +, -, $\times$, $\div$ and balanced parentheses, and nothing else.''} \\

8 &
\textbf{Correct} &
Right numbers \textbf{and} value $=24$ &
\emph{``Correct --- your expression equals 24 using each of [2, 5, 10, 12] exactly once.''} \\

\bottomrule
\end{tabular}
\end{table*}

We further remove solution rationales entirely in \textbf{RGPO-7B (w/o CoT)}. In this setting, failed responses receive only diagnostic feedback from the deterministic task judge. The judge does not reveal a correct expression or a solution trajectory. Instead, it identifies why the current attempt fails, and the policy uses this feedback to generate a revised candidate response. As in standard RGPO, the resulting candidate must satisfy the task verifier before it can provide a learning signal. For 24-Points, the judge first extracts and normalizes the arithmetic expression in the model response and then applies a deterministic sequence of checks. Rather than returning only a binary correctness signal, it produces a short diagnosis of the first detected failure. Table~\ref{tab} summarizes the feedback used by RGPO-7B (w/o CoT). This feedback is intentionally weaker than a solution rationale. For example, when an expression evaluates to 12, the judge tells the model that the result is 12 below the target and should be increased; it does not specify which operation to change or provide a correct expression. Similarly, number-usage errors identify the violated constraint without suggesting a solution. The refinement step must therefore remain policy-generated.
\clearpage
\begin{table}[t]
\centering
\caption{\textbf{RGPO without expert solution rationales on the MM-HELIX 24-Points task.}
RGPO-7B uses only programmatically generated solver traces, whereas RGPO-7B (w/o CoT) receives no solution rationale and refines failed attempts using only deterministic verifier feedback.}
\label{tab:mmhelix}
\small
\begin{tabular}{@{}lrr@{}}
\toprule
\textbf{Model} & \textbf{Accuracy} & \textbf{Response length} \\
\midrule
Qwen2.5-VL-32B-Instruct & 36.67 & 2226.5 \\
WeThink-7B              & 16.67 & 3054.0 \\
VLAA-Thinker-7B         & 10.00 & 3431.1 \\
Llama-3.2-11B-Vision    & 3.33  & 1349.7 \\
MM-HELIX-7B             & 33.33 & 12667.1 \\
\midrule
\textbf{RGPO-7B}        & \textbf{40.00} & 9584.9 \\
RGPO-7B (w/o CoT)       & 36.67 & 10592.8 \\
\bottomrule
\end{tabular}
\end{table}

Despite receiving no expert-written rationale, RGPO-7B achieves the highest accuracy in this comparison at $40.00\%$ (Table \ref{tab:mmhelix}) using only the mechanical solver traces, outperforming MM-HELIX-7B ($33.33\%$). More strikingly, \textbf{RGPO-7B (w/o CoT) reaches $36.67\%$ without any solution rationale}, outperforming MM-HELIX-7B and matching the substantially larger Qwen2.5-VL-32B-Instruct model in this evaluation.  These results suggest that the useful abstraction in RGPO is broader than rationale-prefix supervision. What the refinement stage requires is a \emph{scaffold that makes an improved trajectory easier to discover}; that scaffold can be a partial solution, a mechanically generated trace, or structured feedback describing why the current attempt failed. The final training signal still comes from the policy's own verifier-approved response rather than from directly imitating the guidance source.

\subsection{Effect of Prompt Format}
\label{sec:prompt_format}

In Figure \ref{fig:time}, we show that prompt format has a significant effect on both optimization and efficiency. The Mulberry-style prompt requires image descriptions, rationales, step-by-step sections, additional delimiters, and line breaks, produces much longer responses throughout training, typically around 380--400 tokens. Despite this higher generation cost, its reward remains nearly flat, suggesting that the complex output schema makes the format difficult to learn and does not translate into better task performance without additional supervision (e.g. from SFT loss in the AHPO baseline in the main paper). In contrast, the DeepSeek-R1-style prompt shows a clear reward increase after training begins to improve, while its response length drops sharply to below 100 tokens, indicating that the model learns a more concise and effective reasoning format. This reduction in response length also lowers the runtime per step, making training more efficient. These results motivate our use of the simpler DeepSeek-R1-style prompt: it reduces formatting burden, avoids unnecessary verbosity, and enables more stable learning without requiring extra supervised fine-tuning to enforce the response structure.

\begin{figure}[!ht] 
    \centering
     \includegraphics[width=1\columnwidth]{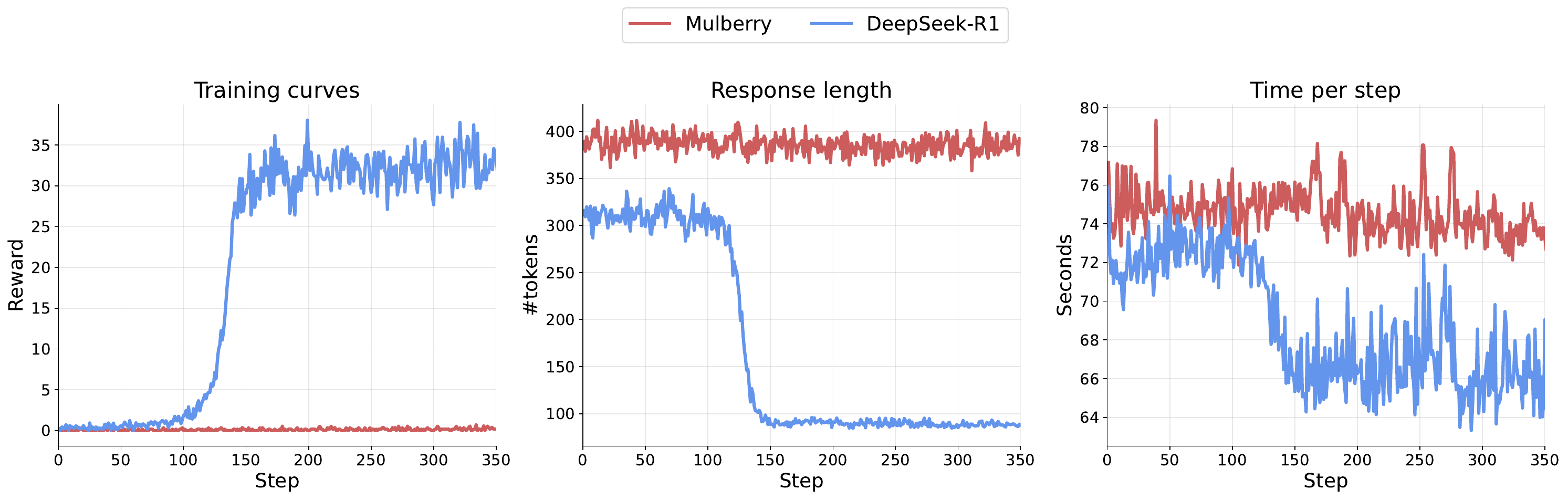}
    \caption{\textbf{Effect of prompt format on training dynamics and efficiency.} We compare the Mulberry-style structured prompt with the simpler DeepSeek-R1-style prompt in terms of reward, response length, and runtime per training step. The DeepSeek-R1-style prompt leads to substantially better learning while producing shorter responses and reducing training time.}
    \label{fig:time}
\end{figure}

\subsection{Robustness to Rationale Quality}

Figure \ref{fig:mmmath} evaluates whether training with rationale guidance leads to improved no-hint performance or merely teaches the model to rely on hints. On the left, both methods are evaluated with the correct CoT appended to the question as a rationale hint. The full-hint baseline achieves the highest accuracy, 95.2, slightly outperforming RGPO at 93.8, which is expected because the baseline is optimized to maximally exploit the provided hint. However, this advantage disappears in the original no-hint setting, where RGPO improves over the full-hint baseline from 33.6 to 38.4. The difference becomes more pronounced under noisy hints, where the CoT is shuffled and appended with an instruction that the model may use it if helpful. In this setting, the full-hint baseline drops to 20.4, suggesting that it tends to trust and follow the misleading rationale rather than solve the problem independently. RGPO achieves a higher score of 25.5, indicating better resistance to corrupted guidance. These results support the motivation of RGPO: rationales should serve as temporary scaffolds during training, not as permanent dependencies. 

\begin{figure}[!ht] 
    \centering
     \includegraphics[width=1\columnwidth]{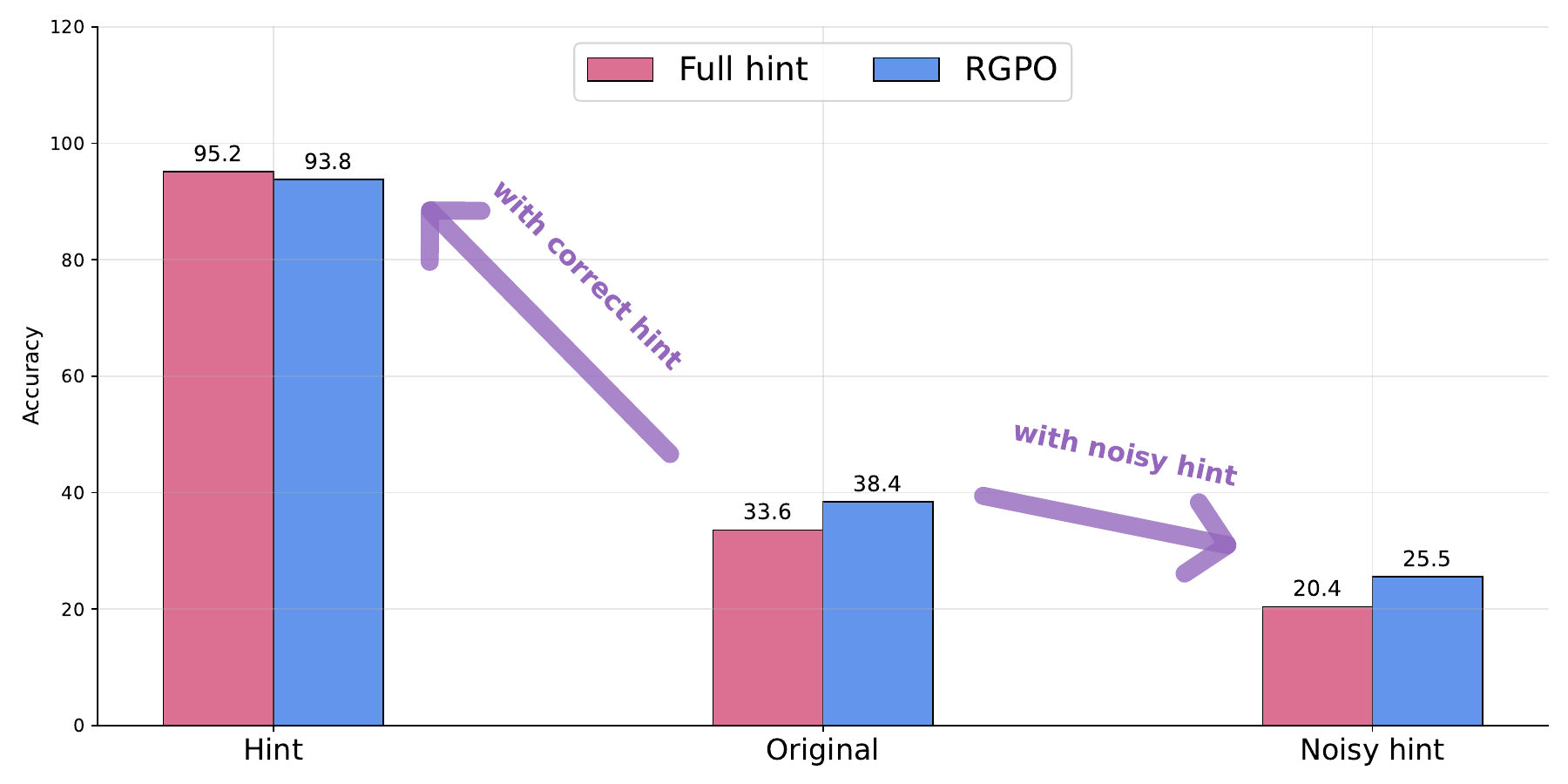}
    \caption{\textbf{Robustness to helpful and noisy rationale hints on MMMath: } We compare RGPO with a full-hint baseline under three evaluation settings: correct rationale hints, the original no-hint benchmark, and shuffled noisy hints. While the full-hint baseline performs slightly better when the correct CoT is provided, RGPO achieves higher accuracy in the original and noisy-hint settings, indicating stronger robustness and less dependence on external rationales.}
    \label{fig:mmmath}
\end{figure}

By selectively using hints only after failure and optimizing for refined solutions, RGPO improves no-hint reasoning while reducing over-reliance on potentially noisy rationales from MMMath~\citep{sun2024mm}. This aligns with recent concerns about over-privileging where they can leak rationales and degrade validation performance, while trained models may overuse unnecessary rationales \citep{yang2026selfdistilledrlvr,li2026fortis}. Always exposing a full or noisy CoT creates a similar shortcut, encouraging the model to rely on the rationale rather than solve independently. RGPO mitigates this risk by providing hints only after unguided failure, verifier-filtering guided outputs, and distilling accepted responses back to the no-hint prompt.

\subsection{Qualitative Analysis}
\label{app:qualitative}

\begin{figure}[!ht] 
    \centering
     \includegraphics[width=1\columnwidth]{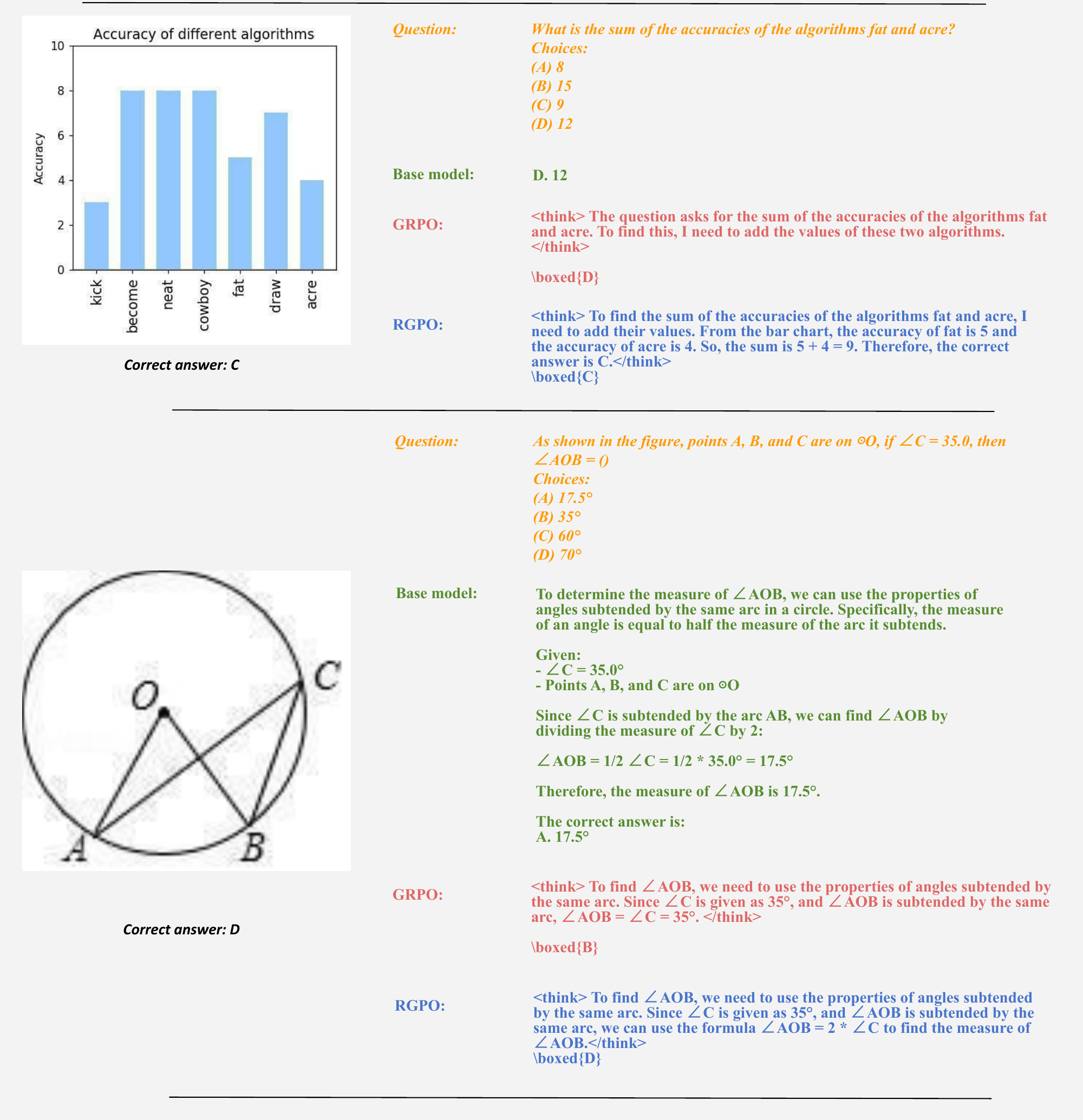}
    \caption{\textbf{Qualitative comparison of model responses:} Example outputs from the base model, GRPO, and RGPO.}
    \label{fig:example1}
\end{figure}

We provide qualitative examples comparing the responses generated by the base model, GRPO, and RGPO in Figure \ref{fig:example1} and Figure \ref{fig:example2}. Compared with the base model, GRPO and RGPO often yield more focused reasoning processes. Moreover, by using rationales as temporary scaffolds, RGPO is able to identify the relevant solution path and answer questions more concisely.

\begin{figure}[!ht] 
    \centering
     \includegraphics[width=1\columnwidth]{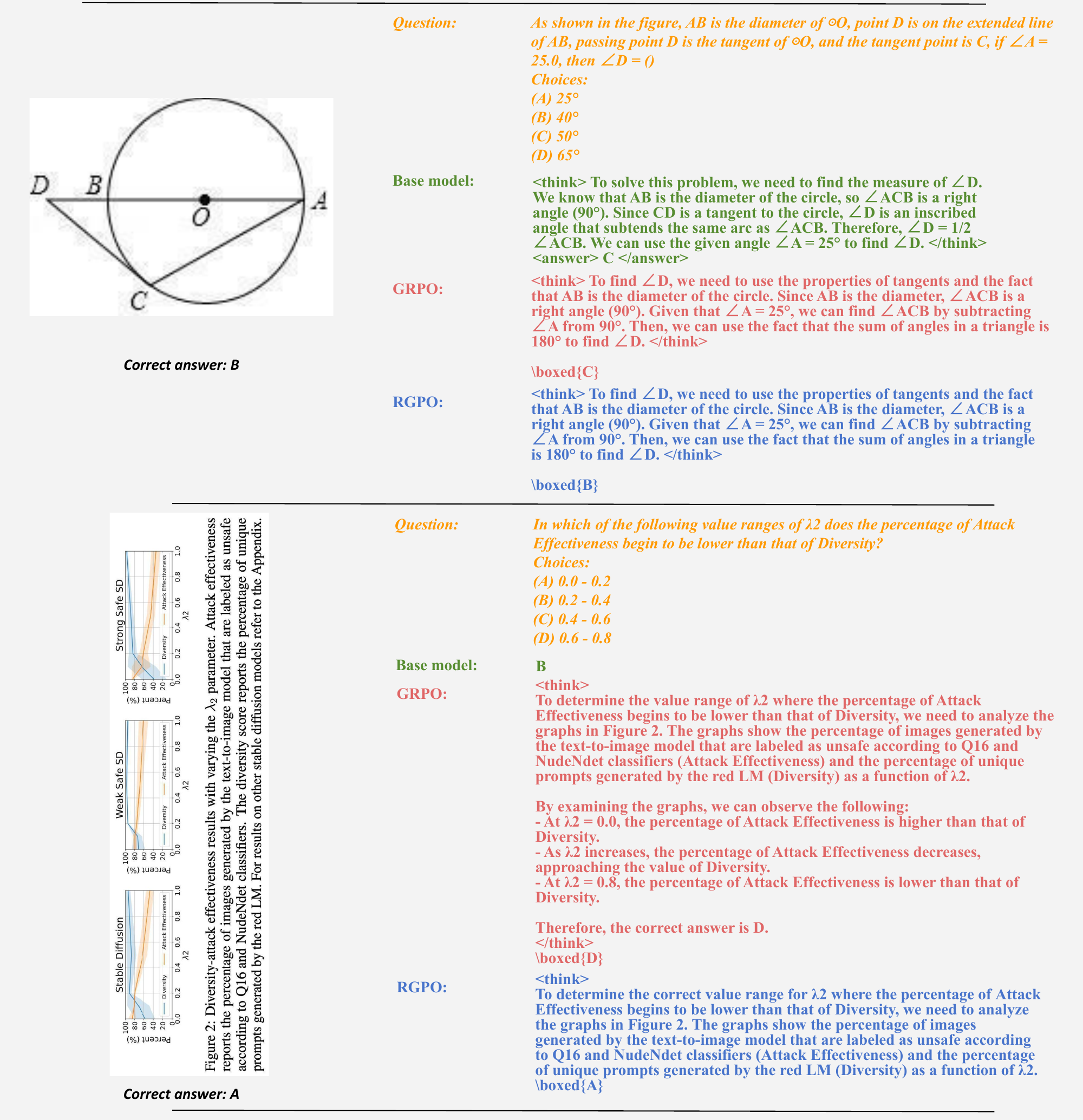}
    \caption{\textbf{Qualitative comparison of model responses:} Example outputs from the base model, GRPO, and RGPO.}
    \label{fig:example2}
\end{figure}

\begin{figure}[!ht] 
    \centering
     \includegraphics[width=1\columnwidth, trim=0.cm 6cm 0.cm 0cm,clip]{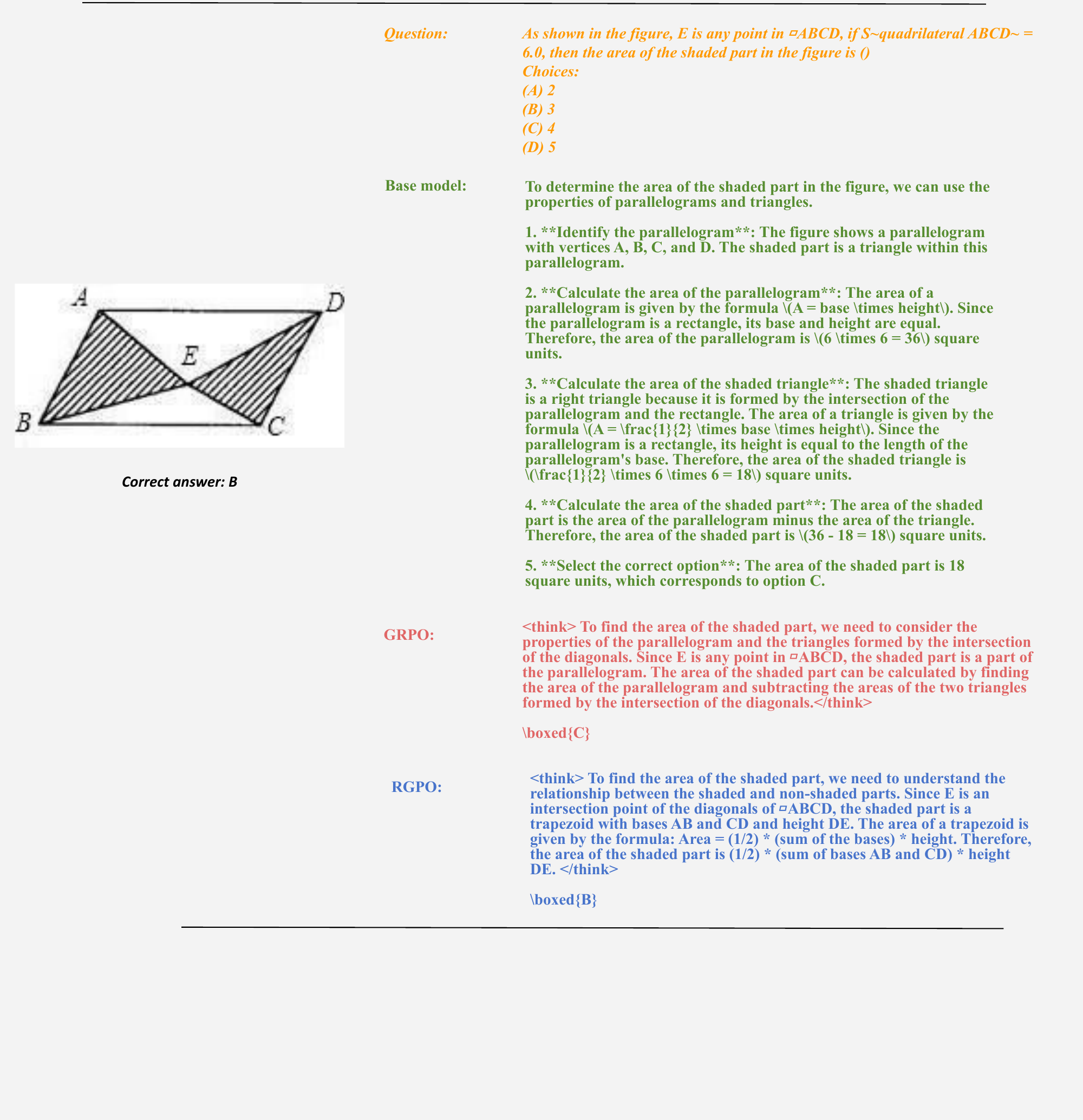}
    \caption{\textbf{Qualitative comparison of model responses:} Example outputs from the base model, GRPO, and RGPO.}
    \label{fig:example3}
\end{figure}

\clearpage
\section{Limitations}
\label{sec:limitations}

In its standard form, RGPO assumes a reference rationale or another informative guidance source, and its effectiveness depends on the quality and granularity of that guidance. The feedback-only MM-HELIX variant relaxes the need for a solution rationale, but requires task-specific diagnostic verifier feedback. If the revealed prefix contains too much answer-specific information, the method may encourage shortcut learning or memorization rather than transferable reasoning; if it contains too little or is stylistically far from the model's own reasoning distribution, it may fail to rescue hard examples from sparse-reward collapse. Moreover, the reward-gated acceptance rule only guarantees improvement under the available verifier, not genuine reasoning quality, so RGPO can still inherit verifier bias, reward hacking, or spurious formatting gains. The adaptive schedule also introduces additional compute from guided rollouts and auxiliary SFT updates, making the method more expensive than a single-pass PPO or GRPO baseline. However,  we keep the number of rollouts per prompt the same as in the GRPO baseline to ensure a fair comparison.

\section{Broader Impacts}
\label{sec:broader_impacts}

RGPO aims to make reasoning-oriented post-training more sample-efficient by using reference rationales as adaptive training-time scaffolds, which may help smaller language and vision-language models learn from sparse verifiable rewards without relying as heavily on large teacher models, rejection sampling, or fully formatted supervised traces; this could benefit education, scientific analysis, software assistance, mathematical problem solving, and multimodal question answering. However, stronger reasoning ability is inherently dual-use: the same improvements that help benign problem solving may also assist harmful planning, automated persuasion, misinformation generation, cyber misuse, or other strategic tasks. RGPO also improves verifier reward rather than guaranteeing faithful reasoning, so generated rationales may remain incomplete, misleading, or optimized for the reward function rather than for truth, especially when verifiers or reference solutions contain biases, shortcuts, or annotation artifacts. These risks suggest that RGPO should be paired with careful verifier design, dataset auditing, robustness and out-of-distribution evaluation, and appropriate deployment safeguards for high-capability models. The method also adds compute overhead through guided rollouts and auxiliary distillation, although this cost may be offset in some settings if adaptive hinting improves sample efficiency or enables smaller models to reach competitive reasoning performance.
\clearpage